\documentclass{article}
\usepackage{iclr2027_conference,times}

\usepackage{amsmath,amssymb,amsthm}
\usepackage{mathtools}
\usepackage{bm}
\usepackage[colorlinks=true,linkcolor=blue,citecolor=blue,urlcolor=blue]{hyperref}
\usepackage{enumitem}
\usepackage{booktabs}
\usepackage{graphicx}
\usepackage{microtype}
\usepackage{xcolor}
\usepackage{array}
\usepackage{listings}
\usepackage{pdflscape}
\usepackage{float}
\usepackage{pgfplots}
\pgfplotsset{compat=1.18}
\usepgfplotslibrary{groupplots}
\usepackage{tikz}
\usetikzlibrary{arrows.meta}
\usepackage{caption}

\definecolor{codeblue}{RGB}{26,65,138}
\definecolor{codegreen}{RGB}{34,112,72}
\definecolor{codered}{RGB}{163,45,45}
\definecolor{codepurple}{RGB}{111,55,145}
\definecolor{codegray}{RGB}{95,99,104}
\lstdefinestyle{paperpython}{
  language=Python,
  basicstyle=\ttfamily\fontsize{5}{5.5}\selectfont,
  keywordstyle=\bfseries\color{codeblue},
  commentstyle=\color{codegreen},
  stringstyle=\color{codered},
  numberstyle=\color{codegray},
  showstringspaces=false,
  keepspaces=true,
  columns=fullflexible,
  breaklines=true,
  breakatwhitespace=false,
  tabsize=2,
  aboveskip=0pt,
  belowskip=0pt,
}
\lstdefinestyle{paperast}{
  basicstyle=\ttfamily\fontsize{5}{5.5}\selectfont,
  keywordstyle=\bfseries\color{codepurple},
  stringstyle=\color{codered},
  morekeywords={Module,Assign,Attribute,BinOp,Call,Compare,Constant,Expr,For,
    FunctionDef,If,Import,Name,Return,Subscript,Tuple,Yield,alias,arg,arguments,
    Add,Div,Eq,Mod,Sub},
  showstringspaces=false,
  keepspaces=true,
  columns=fullflexible,
  breaklines=true,
  breakatwhitespace=false,
  aboveskip=0pt,
  belowskip=0pt,
}

\theoremstyle{plain}
\newtheorem{theorem}{Theorem}
\newtheorem{proposition}{Proposition}

\newtheorem{lemma}{Lemma}

\theoremstyle{definition}

\newtheorem{assumption}{Assumption}
\newtheorem{example}{Example}

\theoremstyle{remark}
\newtheorem{remark}{Remark}

\makeatletter
\renewenvironment{proof}[1][\proofname]{
  \par
  \pushQED{\qed}
  \normalfont
  \topsep6\p@\@plus6\p@\relax
  \trivlist
  \item[\hskip\labelsep\bfseries #1\@addpunct{.}]
  \ignorespaces
}{
  \popQED
  \endtrivlist
  \@endpefalse
}
\makeatother

\DeclareMathOperator*{\argmax}{arg\,max}
\DeclareMathOperator{\Tr}{Tr}
\DeclareMathOperator{\Cov}{Cov}
\DeclareMathOperator{\Var}{Var}
\DeclareMathOperator{\TV}{TV}

\DeclareMathOperator{\softmax}{softmax}
\newcommand{\E}{\mathbb{E}}

\newcommand{\hth}{\hat\theta}

\title{Fine-Tuning Fixes Mode Collapse and Over-Dispersion in LLMs}

\author{%
  Kirill Skobelev \\
  Northwestern University \\
  {\small\texttt{kirill@u.northwestern.edu}} \\
  \And
  Eric Fithian \\
  University of Chicago \\
  {\small\texttt{efithian@uchicago.edu}} \\
  \And
  X.Y. Han \\
  University of Chicago \\
  {\small\texttt{xyhan@uchicago.edu}} \\
}

\iclrfinalcopy
\hypersetup{pdfauthor={Kirill Skobelev, Eric Fithian, X.Y. Han}}

\AddToHook{build/column/before}{%
  \ifnum\value{page}=\getpagerefnumber{fig:exp4fullmatrix}\relax
    \raggedbottom
  \fi
}
\graphicspath{{./}{figs/}}

\begin{document}
\maketitle
\fancyhead{}
\renewcommand{\headrulewidth}{0pt}

\begin{abstract}
Recent work by \citet{doshi2024}, \citet{bisbee2024}, and \citet{xie2026} raises concerns that outputs from large language models (LLMs) tend to be under-diverse: they repeat or resemble one another more often than responses from the population they are meant to represent, a phenomenon known as mode collapse. In this work, we show that whether mode-collapse---or its opposite---occurs depends on the specific model and dataset used. Further, with sufficient supervised fine-tuning (SFT) data, LLM output diversity converges toward that of the target distribution from which fine-tuning data are sampled. To quantify this comparison, we measure the probability that two responses sampled independently conditional on the same fixed prompt coincide (collide). In one set of experiments, we use exact token sequences; in the other, we use the generalised version of that---the expected similarity between responses under a kernel. We define miscalibration as a nonzero model-minus-target collision gap. We derive a bias--variance decomposition of the expected gap between the model's and target's collision probabilities, showing that SFT is not inherently biased toward mode collapse or its opposite: finite-sample SFT can leave a model either under- or over-dispersed, depending on the model and dataset. Finally, we show that the absolute gap is bounded by the square root of the Kullback--Leibler (KL) divergence from the target distribution to the model. Consequently, a model sufficiently close to optimal under population cross-entropy cannot exhibit arbitrarily miscalibrated diversity.
We test the decomposition and the bound in three experiments: (1) we fit 100 small transformers at each of 100 log-spaced sample sizes on each of two synthetic order-16 languages; (2) we fine-tune four LLMs using low-rank adaptation on responses from the General Social Survey, American National Election Studies, and World Values Survey; and (3) we repeat Experiment~(2) on CodeNet, a dataset of human solutions to coding tasks, measuring program similarity with normalized Zhang--Shasha edit distance between canonical abstract syntax trees. We find substantial heterogeneity in over- and under-diversity across models and datasets. More target data moves model diversity toward the human (or synthetic target) level in all experiments, consistent with our theoretical predictions.
These results show that diversity miscalibration can arise from finite-sample error and shrink as SFT better approximates the target distribution. Accordingly, as the sample size of target-distribution data increases, model diversity moves toward the target level.
\end{abstract}

\section{Introduction}
\label{sec:intro}

Large language models (LLMs) are increasingly used as simulated survey respondents \citep{argyle2023}, for creative writing \citep{chakrabarty2024}, as brainstorming partners \citep{girotra2023}, and in other domains where faithfulness of the LLM's outputs to the diversity of the target population is important. Empirical studies show that LLMs often exhibit a narrow range of outputs, a phenomenon known as mode collapse: the model's outputs repeat, or resemble one another, more often than samples from the population they are meant to represent. For instance, stories written with artificial intelligence (AI) assistance tend to be similar to one another \citep{doshi2024}; AI-assisted users may generate ideas that are less semantically distinct \citep{anderson2024}; and synthetic survey respondents can reproduce aggregate patterns while suppressing individual and group-level variation \citep{bisbee2024,wang2025,xie2026}.

These findings are often interpreted as evidence that low output diversity is an inherent limitation of LLMs \citep{xie2025}. However, when a pre-trained or instruction-tuned model is fine-tuned on a new dataset, its starting distribution is being compared with a target distribution it has not yet learned. It is unclear whether this under-diversity persists as a model learns the target distribution more accurately. If mode collapse is inherent, increasing the amount of target-distribution fine-tuning data need not eliminate it; if, instead, mode collapse and other distribution miscalibrations simply reflect finite-sample error, model diversity should converge toward the target level with sufficient fine-tuning. To settle this distinction, we begin with formalizing it theoretically.


We quantify LLM output diversity by collision probability---the probability that two responses sampled independently conditional on the same fixed prompt coincide---or, more generally, by their expected similarity under a kernel (Section~\ref{sec:model}). For next-token comparisons, we instead fix the full context: the prompt and token prefix. Model and target are compared under the same conditioning and similarity measure. Their collision gap is the model's collision probability minus the target's; where zero indicates calibrated diversity, a positive gap indicates mode collapse, and a negative gap indicates over-dispersion.

This formalization allows us to make three claims (Figure~\ref{fig:overview-improved}). (1)~SFT does not inherently produce under- or over-dispersion; the direction depends on the model and dataset. The difference in collision probabilities, averaged over independent fits, decomposes into nonnegative variance and squared-bias terms, which promote concentration, and a sign-indefinite target--bias alignment term that can offset them (Section~\ref{sec:neural}).
(2)~Magnitude: the absolute collision gap is bounded by the square root of the Kullback--Leibler (KL) divergence from the target distribution to the model (Theorem~\ref{thm:kernelstability}). This bound implies (3)~Calibration in the limit: as this KL divergence approaches zero, the collision gap necessarily approaches zero. In practice, this leads to the intuitive conclusion that if more target data and loss minimization bring population cross-entropy toward the value achieved by the target distribution itself, model diversity must approach the target level.

\begin{figure}[t]
\centering
\includegraphics[width=\textwidth]{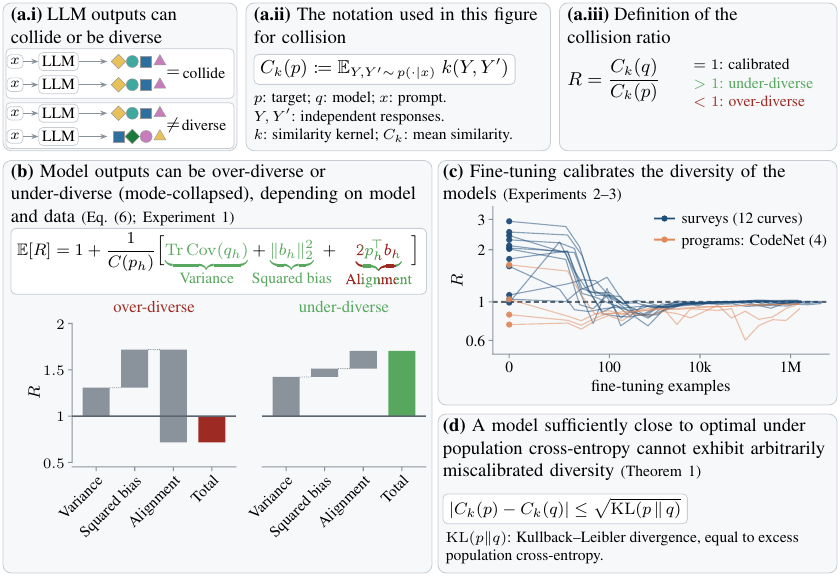}
\caption{\textbf{Large Langauge Models (LLMs) can produce diverse or colliding outputs given the same prompt. The extent to LLMs do so can be measured. Higher average output similarity than the target's is referred to as mode-collapse. However, contrary to wide-spread belief, mode-collapse is not guaranteed, and depending on model and dataset, models can instead be over-diverse (the opposite of mode-collapse). It can be shown that fine-tuning calibrates the models, that is, they become not too under- and over-diverse.}
\textbf{(a.i)} Independent outputs from the same LLM and prompt can collide or differ.
\textbf{(a.ii)} The notation used in this figure for collision.
\textbf{(a.iii)} The definition of the collision ratio $R$.
\textbf{(b)} Shows the decomposition of Eq.~\eqref{eq:ifbox} using data from Experiment~1 (\S\ref{sec:exp-synth}). The two plots correspond to two different model and data set ups. The plots demonstrate that Alignment term of Eq.~\eqref{eq:ifbox} can make a model either under- or over-diverse. Here $b_h=\E[q_h]-p_h$ is the bias at context $h$, with $\E$ over independent fits. The plotted curves average contexts and start at $R=1$, with each contribution divided by $C(p_h)$. \textbf{(c)} Plots median collision ratio $R$ for early-stopped survey and final CodeNet adapters against fine-tuning examples from Experiments~2--3 (\S\S\ref{sec:exp-survey}--\ref{sec:exp-code}). The experiments show that additional fine-tuning leads to diversity calibration.
\textbf{(d)} Theorem~\ref{thm:kernelstability} bounds the absolute model--target collision gap by the square root of their Kullback-Leibler divergence. Consequently, a model sufficiently close to optimal under population cross-entropy cannot exhibit arbitrarily miscalibrated diversity.}
\label{fig:overview-improved}
\end{figure}

To support our claims empirically, we set up three experiments (Section~\ref{sec:empirical}). Experiment~1 tests claim~(1) and checks the bound of claim~(2): in two synthetic order-$16$ language settings with exactly computable targets, models with the generative pre-trained transformer (GPT) architecture fitted across $100$ log-spaced sample sizes produce collision distortions of both signs (Section~\ref{sec:exp-synth}, Figure~\ref{fig:exp2decomp}), and the bound holds on every evaluated context (Appendix~\ref{sec:bound-check}, Figure~\ref{fig:exp2bound}). Experiments~2 and~3 test the empirical prediction motivated by Theorem~\ref{thm:kernelstability} that increasing the amount of fine-tuning data improves collision calibration. We use low-rank adaptation (LoRA) to fine-tune four instruction-tuned LLMs (\texttt{gemma-2-2b-it} \citep{gemma_2}, \texttt{gemma-3-4b-it} \citep{gemma_3}, \texttt{Qwen3.5-2B} and \texttt{Qwen3.5-4B} \citep{qwen}) on General Social Survey (GSS) \citep{davern2025}, American National Election Studies (ANES) \citep{anes2021}, and World Values Survey (WVS) \citep{haerpfer2020}. Three of the four base models have median collision ratios of $1.6$ to $2.9$ against the unweighted survey references. With sufficient SFT data, all four models bring their median collision ratios close to $1$ across all three surveys (Section~\ref{sec:exp-survey}, Figure~\ref{fig:exp4fullmatrix}). Next, we apply the same method to CodeNet \citep{puri2021codenet}, a dataset of human solutions to coding tasks, where we measure distance between solutions using the normalized Zhang--Shasha similarity kernel \citep{zhangshasha1989}. Under the evaluated decoding settings, the base models start on both sides of the accepted-solution reference ($R=0.74$--$1.64$). Supervised fine-tuning moves the median structural collision ratio toward $1$ from both directions. This improvement holds for all four models among parseable generations when each model is evaluated on the same problems before and after fine-tuning.

\section{Related work}
\label{sec:related}

\noindent {\bf Evidence of output mode collapse.} Generative AI can enhance individual story ratings and increase similarity between AI-assisted stories \citep{doshi2024}, while co-writing with InstructGPT reduces content diversity in essays \citep{padmakumar2023}. ChatGPT users generate more numerous and detailed ideas than users of another creativity-support tool, but these ideas exhibit less semantic distinctiveness across users \citep{anderson2024}. \citet{sourati2025} argue that LLMs reinforce dominant styles of language and reasoning, and that widespread reliance on the same few models amplifies this convergence. Our results show that SFT is not inherently biased toward mode collapse or its opposite: the direction of miscalibration depends on the model and dataset, and fine-tuning can bring diversity toward the target from either direction. LLMs are also employed as proxies for human participants in ``silicon sampling'' studies \citep{argyle2023}, but recent work shows that synthetic surveys can match aggregate statistics while failing to capture individual variation and regression structures \citep{bisbee2024}, and can misrepresent and flatten identity groups \citep{wang2025}. A  benchmark of 15 LLMs across seven surveys documents the same compression of heterogeneity \citep{xie2026}. Recursive fitting on model-generated data also removes data points from the tails of the original distribution \citep{shumailov2024,seddik2024}, illustrating the ease with which generative resampling can erase rare modes.

\noindent {\bf Preference optimization and diversity.} Preference optimization has been shown to shrink output diversity. Reinforcement learning from human feedback (RLHF) with proximal policy optimization (PPO) reduces both within- and across-prompt diversity for summarization, as assessed by multiple lexical and semantic metrics, though the same study finds no meaningful diversity differences for instruction following \citep{kirk2024}. The stage at which this contraction occurs varies across model lineages, with some models showing the largest reduction during SFT and others during direct preference optimization (DPO) \citep{karouzos2026}. These effects are linked to preference data and annotator heterogeneity. Annotators tend to favor responses to which base models assign higher likelihood even when responses are matched on correctness; consequently, a perfectly estimated preference reward can preserve typicality bias toward conventional responses \citep{verbalized2025}. Finally, aggregating annotators with heterogeneous criteria implicitly implements the Borda-count rule rather than recovering an observer-independent utility \citep{siththaranjan2024}, and a single scalar reward may not represent multimodal population preferences \citep{chakraborty2024}. For a reference policy $\pi_{\mathrm{ref}}$, population reward $r^\star$, and KL penalty strength $\beta$, the optimal policy has the form $\pi^\star(y\mid x)\propto\pi_{\mathrm{ref}}(y\mid x)\exp\!\big(r^\star(x,y)/\beta\big)$, and DPO directly fits a policy within this implicit-reward class \citep{rafailov2023}. Finally, common reward structures and regularization regimes can produce unimodal or concentrated optima under both forward and reverse KL divergence, indicating that the KL penalty alone does not preserve the reference distribution's diversity \citep{gxchen2025}. These population-level mechanisms persist as preference data size increases. We focus instead on the finite-sample distortion around the population objective, adding a finite-sample component to the objective-induced gap $C(\pi^\star)-C(p)$ from existing literature (Appendix~\ref{sec:rl}). Because form-sensitive and content-sensitive diversity metrics capture different aspects of generator behavior \citep{tevet2021}, our conclusions are metric-specific.

\noindent {\bf Relation to prior work.} \citet{xie2026} find compressed heterogeneity in survey microdata generated by 15 LLMs; fine-tuning one model on 1{,}000 records improves aggregate realism on held-out surveys. Our analysis shows that the collision gap can have either sign and must vanish as excess population cross-entropy approaches zero (\S\ref{sec:neural}, Theorem~\ref{thm:kernelstability}). Whereas the entropy-regularized game-theoretic method GEM \citep{li2025}, Selectively Encouraging Diversity in Supervised Fine-Tuning (SED-SFT) \citep{chen2026}, and tempered focal (TOFU) loss \citep{klypa2026} modify the SFT objective to protect diversity, our experiments show that plain SFT moves diversity toward the target across four models, three surveys, and CodeNet, from both directions. \citet{banayeeanzade2026} attribute sequence-level diversity loss to miscalibrated token probabilities; our decomposition and bound hold for any fitting procedure.

\section{Theory}
\label{sec:model}

Let $V$ be the token vocabulary, and $V^*$ the set of finite token sequences. Let prompts $x$ be drawn from a distribution $\nu$, and let $p(y\mid x)$ be the target distribution over responses $y\in V^*$. The model outputs token $a_t$ given context $h_t := (x, y_{<t})$, where $y_{<t}$ is the response prefix before step $t$. The true conditional distribution is $p_h(a) := p(a \mid h)$. For an LLM with parameters $\theta$, $f_\theta(h)$ is the vector of pre-softmax scores (logits) predicted by the model given context $h$. Applying softmax gives $q_\theta(a\mid h):=\softmax(f_\theta(h))_a$, the model's probability of token $a$ given context $h$. SFT fits $\theta$ by minimizing empirical cross-entropy, the average of $-\log q_\theta(a\mid h)$ over observed context--token pairs $(h,a)$ in the fine-tuning data. A model fit produces parameters $\hth$ and model $q_{\hth}$.

We now define collision for an arbitrary distribution $\pi$ over tokens. Collision is the probability that two tokens drawn independently from $\pi$ are equal:
\begin{equation}
\textstyle c(\pi) := \sum_{a\in V} \pi(a)^2.
\label{eq:local-collision}
\end{equation}

The inverse Simpson index is the number of equally likely choices that would give the same collision probability. For $K$ equally likely tokens, it equals $K$. For target and model distributions $p$ and $q$, define
\begin{equation}
\textstyle N_2(p) := \frac{1}{C(p)}, \qquad N_2(q) := \frac{1}{C(q)},
\end{equation}
where $C$ denotes collision. At a fixed context $h$, $C(p)=c(p_h)$ and $C(q)=c(q_\theta(\cdot\mid h))$.

We compare model and target diversity through the collision ratio
\begin{equation}
\textstyle R := \frac{C(q)}{C(p)} = \frac{N_2(p)}{N_2(q)},
\label{eq:relative-diversity}
\end{equation}
where $C(p)>0$: $R=1$ indicates calibrated diversity, $R>1$ mode collapse, and $R<1$ over-dispersion.

For a fixed prompt $x$, draw complete responses $Y$ and $Y'$ independently from $p(\cdot\mid x)$. Their joint law is the product measure $p( . | x) \otimes p(. | x)$; independence is conditional on $x$, and averaging over a shared random prompt need not preserve it. Under a similarity kernel $k(y,y') \in [0,1]$, kernel collision is
\begin{equation}
C_{p,k}(x) = \E_{Y,Y'\sim p(\cdot\mid x)}[k(Y,Y')].
\label{eq:kernel-collision}
\end{equation}

The kernel function $k(y, y')$ must satisfy $k(y, y') \in [0, 1]$, $k(y, y) = 1$, and $C_{p, k}(x) > 0$. A kernel representing exact collision is defined as $k(y, y') = \mathbf{1}\{y = y'\}$. Finally, task-specific kernels may measure embedding cosine similarity, shared cluster membership, normalized Zhang--Shasha \citep{zhangshasha1989} similarity between code abstract syntax trees (ASTs), or similarity of survey responses, with all values shifted to the range $[0, 1]$.

\subsection{Sign and magnitude of the difference in collision probabilities}
\label{sec:neural}

\noindent {\bf The collision decomposition.} For a fixed context $h$, consider fitted conditional probability $q_{\hth,h}=q_{\hth}(\cdot\mid h)$ and true conditional probability $p_h$. Applying Eq.~\eqref{eq:kernel-collision} to next-token outputs at context $h$, with the exact-match kernel $k(a,b)=\mathbf{1}\{a=b\}$, gives $C_h(q):=\sum_{a\in V}q(a\mid h)^2=c(q(\cdot\mid h))$. Thus $C_h(q_{\hth})=\|q_{\hth,h}\|_2^2$ and $C_h(p)=\|p_h\|_2^2$. We call the difference $\E[C_h(q_{\hth})] - C_h(p)$ the expected collision gap: it is positive when the model collides more often than the target (mode collapse) and negative when the model is more diverse. For an arbitrary random vector $X$, we use the standard variance decomposition identity: $\E\|X\|_2^2 = \|\E X\|_2^2 + \Tr\Cov(X)$. Here $\Tr\Cov(X)$ is the sum of the component variances. Taking $X=q_{\hth,h}$, $\E[X]$ is the mean fitted distribution across independent fits. Next, define its bias relative to the target as
\begin{equation}
\label{eq:bias}
b_h := \E[q_{\hth,h}] - p_h.
\end{equation}
Thus $\E[q_{\hth,h}]=p_h+b_h$. Substituting this into the variance decomposition identity, expanding $\|p_h+b_h\|_2^2$, and dividing by $C_h(p)$ gives
\begin{equation}
\label{eq:ifbox}
\boxed{\textstyle\;
\E[R] = 1 + \frac{1}{C_h(p)}\Big[
\underbrace{\Tr\Cov(q_{\hth,h})}_{\text{variance}\,\ge 0}
\;+\;\underbrace{\|b_h\|_2^2}_{\text{squared bias}\,\ge 0}
\;+\;\underbrace{2\,p_h^\top b_h}_{\text{target--bias alignment}}
\Big].\;}
\end{equation}
The covariance trace and squared bias are nonnegative, and thus contribute to increasing the expected collision ratio. The target--bias alignment term $2p_h^\top b_h$ (referred to as the ``cross term'' in what follows) can be either positive or negative and may dominate these terms, so finite samples alone do not imply mode collapse. Thus the collision ratio $R$ can be above $1$ (mode collapse) or below $1$ (over-dispersion). If $b_h=0$, the expected gap is $\Tr\Cov(q_{\hth,h})\ge0$: the model is calibrated or under-dispersed in expectation, with under-dispersion whenever this variance is positive.

\noindent {\bf The sign of the cross term.} Let $q_{0, h}$ denote the starting distribution at context $h$. For a pre-trained LLM, $q_{0,h}$ is its output distribution given context $h$ before any fine-tuning. At this starting point, using the definitions of bias~\eqref{eq:bias} and collision~\eqref{eq:local-collision} we get
\begin{equation}
2\, p_h^\top b_h =
2\,p_h^\top (q_{0, h} - p_h) \;=\; 2\big(p_h^\top q_{0, h} - C(p_h)\big).
\label{eq:crosslambda}
\end{equation}
The alignment term~\eqref{eq:crosslambda} is positive when the starting distribution overlaps the target's modes more than the target overlaps itself. Moreover, the alignment term can determine whether the expected collision gap is positive or negative.

Using this setup, we can prove the following result bounding the collision gap between distributions $p$ and $q$:

\begin{theorem}[Kernel-collision stability]\label{thm:kernelstability}
Consider the similarity kernel $k:V^*\times V^*\to[0,1]$, where $V^*$ is the set of finite token sequences. Write the collision metric $C_k(\pi)=\sum_{a,b}\pi(a)\pi(b)k(a,b)$, $D_k(\pi)=1/C_k(\pi)$, where $\pi$ is a probability distribution on $V^*$ and $D_k(\pi)$ is its kernel effective diversity. Write $\TV$ for total variation distance, $\otimes$ for the product measure operator, and $\mathrm{KL}(p\|q)=\E_{a\sim p}[-\log q(a)]-\E_{a\sim p}[-\log p(a)]$ for excess population cross-entropy, using natural logarithms. For all distributions $p,q$ on $V^*$,
\begin{equation}
\boxed{\;|C_k(p)-C_k(q)|\;\le\;\TV(p\otimes p,\,q\otimes q)\;\le\;\sqrt{\mathrm{KL}(p\,\|\,q)}\;.}
\label{eq:cebound}
\end{equation}
Consequently, if $\mathrm{KL}(p\,\|\,q)\le\varepsilon$ then $\max\{0,C_k(p)-\sqrt\varepsilon\}\le C_k(q)\le\min\{1,C_k(p)+\sqrt{\varepsilon}\}$, and the kernel effective diversity obeys
\begin{equation}
\textstyle D_k(q)\;\ge\;\frac{1}{C_k(p)+\sqrt{\varepsilon}}\;=\;\frac{D_k(p)}{1+D_k(p)\sqrt{\varepsilon}}\;.
\end{equation}
\end{theorem}

The proof, via total variation on product measures and Pinsker's inequality, is in Appendix~\ref{app:proofs}, together with a sharper collision-specific bound (Proposition~\ref{prop:tightbound}).

\section{Experimental results}
\label{sec:empirical}
To test the bias-variance decomposition and convergence results in the above theory, we now run three experiments. The first uses synthetic languages with known ground truth for each decomposition term, allowing direct measurement of every term. The second and third use human data---survey responses from the GSS \citep{davern2025}, WVS \citep{haerpfer2020}, and ANES \citep{anes2021}, and human solutions to programming problems from Project CodeNet \citep{puri2021codenet}. We report the collision ratio $R$ defined in Eq.~\eqref{eq:relative-diversity}; for example, $R=2$ means the model offers half as many effective choices as the data. Experiment~1 additionally decomposes the expected normalized collision gap $\E[R]-1$ into the three terms of Eq.~\eqref{eq:ifbox}.

\subsection{Experiment 1: synthetic languages with known target distribution}
\label{sec:exp-synth}

We sample data from a synthetic language, fit a series of small GPT models on that data, and compare the diversity of the fitted models against the diversity of the language itself. Unlike experiments using human-generated data, in this setting, target probabilities are known. That allows us to repeat model fits to estimate variance, squared-bias, and target--bias alignment terms in Eq.~\eqref{eq:ifbox}. As we have established in the previous section, the latter term determines whether LLMs mode-collapses or is over-diverse. In order to control this term, we initialize the GPTs using two settings: (a) pre-training from random initialization, and (b) fine-tuning a model pre-trained to favor a small set of tokens. 

In setting~(a), the language has a vocabulary of $1024$ tokens, and each next token depends only on the preceding $16$ tokens. The target probabilities are known exactly at every context, and a small GPT can represent the generating distribution exactly, so model capacity does not limit the fit. In setting~(b), the language combines fixed token frequencies with adjustments based on the preceding $16$ tokens. Appendix~\ref{app:synthetic-data} gives the language construction and sampling details. The two settings differ in the model's starting next-token distribution, model size, and target language. They are not intended as a controlled ablation of the starting distribution alone. Figure~\ref{fig:exp2decomp} shows pre-training from random initialization in panel~(a) and fine-tuning from the pre-trained model in panel~(b).

Decomposing the expected collision gap as in Section~\ref{sec:neural} requires independent fits: a single fitted model does not show the variance across runs. We therefore fit $100$ models at each of $100$ log-spaced sample sizes of $N$ sampled sequences, giving $10{,}000$ GPTs per language. Each seed draws its own dataset, initialization, and minibatch stream, ensuring that the variance term $\Tr\Cov(q_{\hth,h})$ in Eq.~\eqref{eq:ifbox} reflects total across-run fluctuation. Details of model sizes, data construction, and fitting configurations are provided in Appendix~\ref{app:synthetic-data}.

We consider two settings that sit on opposite sides of the initial-alignment condition in~\eqref{eq:crosslambda}. In setting~(a), a GPT starts from random initialization. Its initial next-token distribution $q_{0,h}$ at context $h$ has high entropy; we call this the \textit{diffuse prior}. Its probabilities depend on the context, but the network is randomly initialized, so it does not favor or suppress particular tokens based on the target distribution $p_h$. In setting~(b), we pre-train a GPT to concentrate its probability on a fixed set of common target-language tokens. We call this concentrated starting distribution the \textit{mode-aligned prior}. Appendix~\ref{app:synthetic-data} describes the construction of both priors. This mode-aligned prior is arguably not a realistic model of pre-training from scratch, however, it could showcase the behavior of an existing LLM pre-trained on data from a distribution similar but not identical to the target.


Under the diffuse prior (Figure~\ref{fig:exp2decomp}(a)), the normalized cross term $2p_h^\top b_h/C(p_h)$, averaged over contexts, is negative at every sample size, and is largest in magnitude at small $N$. The normalized variance and squared-bias terms stay positive. The sign of $\E[R]-1$ depends on $N$: at the smallest $N$ the cross term dominates and the measured $\E[R]-1$ is negative (over-dispersion); however, over most models, the two nonnegative terms dominate and $\E[R]-1$ is positive (mode-collapse).

Under the mode-aligned prior, the cross term increases the expected collision ratio at intermediate sample sizes (Figure~\ref{fig:exp2decomp}(b)). Even where positive, it accounts for at most $31\%$ of the measured $\E[R]-1$, so variance and squared bias still account for most of the expected normalized collision gap. At the smallest $N$, however, the cross term is negative and mode collapse is driven by the two nonnegative terms.

To relate these outcomes to the starting distributions, Table~\ref{tab:preft} in Appendix~\ref{app:initial-alignment} checks the initial-alignment condition in Eq.~\eqref{eq:crosslambda} using the mean initial distribution $\bar q_{0,h}:=\E[q_{0,h}]$, estimated across the $100$ seeds. The table reports averages over contexts after the indicated normalization. The normalized cross term is initially negative under the diffuse prior and positive under the mode-aligned prior, but fine-tuning does not preserve these signs at every $N$. In the next two experiments (Sections~\ref{sec:exp-survey} and~\ref{sec:exp-code}), we use human data to test whether fine-tuning moves $R$ toward $1$ (as Theorem~\ref{thm:kernelstability} predicts when excess population cross-entropy approaches zero).

\begingroup
\setlength{\intextsep}{0pt}
\begin{figure}[H]
\centering
\includegraphics[width=\textwidth]{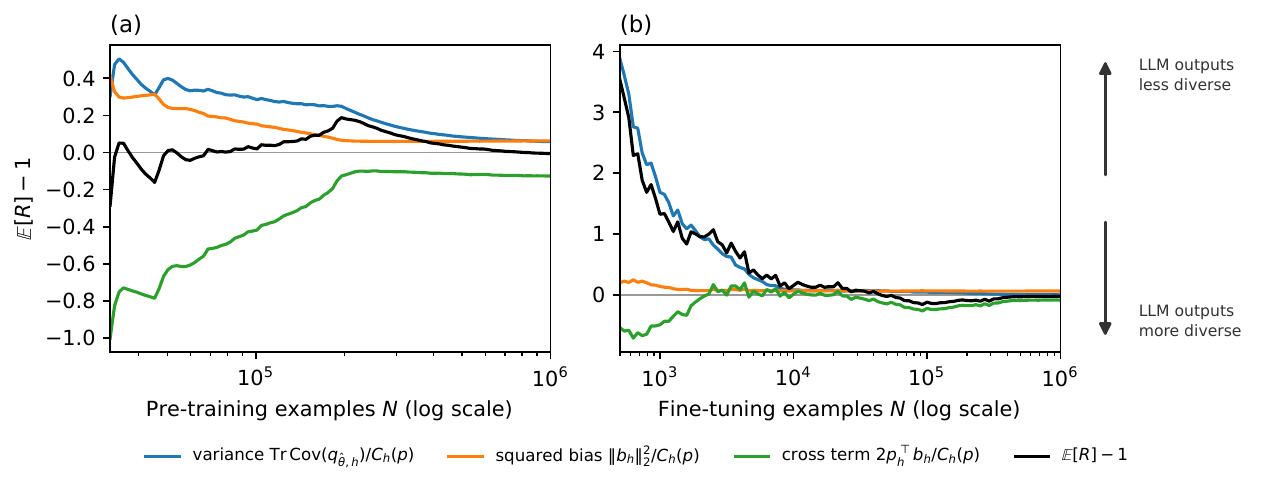}
\caption{\textbf{Decomposition of the normalized gap between model and target collision probabilities.} The two synthetic languages have known target distributions (Appendix~\ref{app:synthetic-data}), so repeated fits let us estimate the contributions of variance, squared bias, and target--bias alignment to $\E[R]-1$ using Eq.~\eqref{eq:ifbox}. Contributions are averaged over contexts and plotted against sampled sequences $N$. (a) shows pre-training from the diffuse prior ($6.135$M-parameter GPT, $N$ from $32$k to $1$M, $2000$ common contexts); (b) shows fine-tuning from the mode-aligned prior ($27.97$M-parameter GPT, $N$ from $500$ to $10^6$). The black line shows the estimated $\E[R]-1$, averaged over contexts: positive values indicate mode collapse and negative values over-dispersion. The three contributions sum to this value. The two panels carry independent vertical and horizontal scales.}
\label{fig:exp2decomp}
\end{figure}
\endgroup

\subsection{Experiment 2: sociological surveys (GSS, WVS, and ANES)}
\label{sec:exp-survey}

In this experiment, we LoRA fine-tune four LLMs on human responses to three large opinion surveys, to test whether supervised fine-tuning moves a model's answer diversity toward the diversity of the population it is asked to imitate through fine-tuning.

We consider three social surveys in this section. First, the General Social Survey (GSS) \citep{davern2025}, a nationally representative survey of U.S. adults since 1972. We use the cumulative dataset spanning 1972--2024, including data from 74,485 respondents with complete demographic information. Second, we use a random subset of one-third of the countries and one-third of the questions from wave 7 of the World Values Survey (WVS) \citep{haerpfer2020}, a cross-national survey of values (the sub-samples are used in the interest of saving compute). Third, we use data from the American National Election Studies (ANES) \citep{anes2021}, a long-running U.S. political survey.

\begin{figure}[tbp]
\centering
\includegraphics[width=\textwidth]{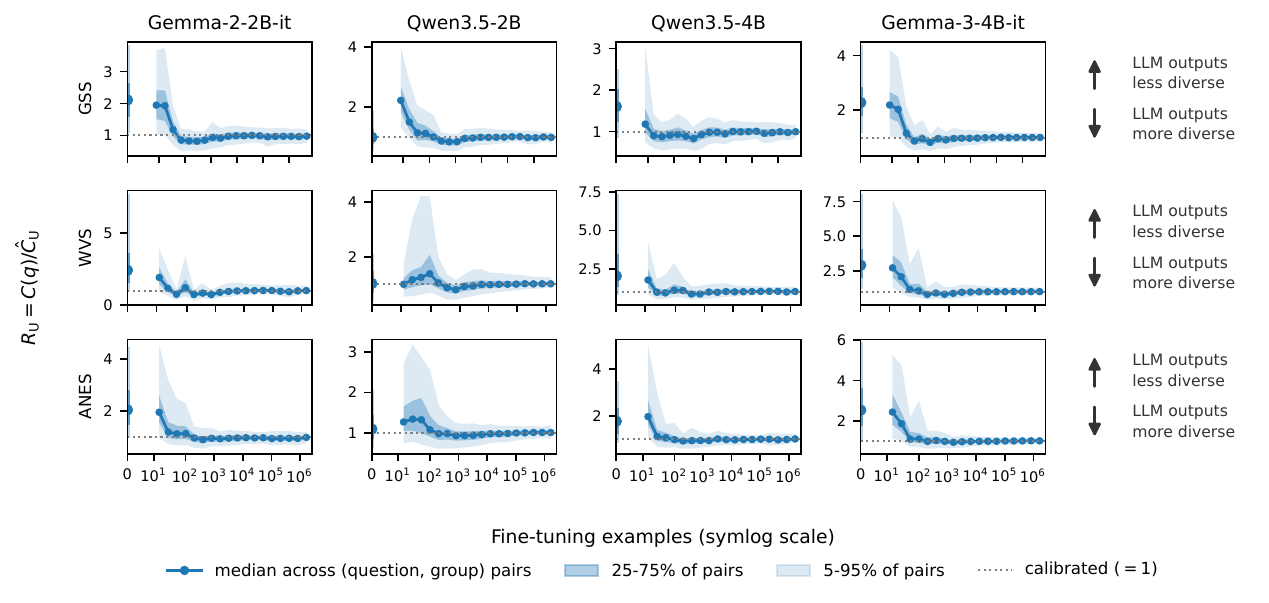}
\caption{\textbf{Experiment~2.} Collision ratio $R_{\mathrm U}=C(q)/\hat C_{\mathrm U}$ ($1$ = calibrated) against fine-tuning examples ($20$ sample sizes for the GSS, $18$ each for the WVS and ANES). $\hat C_{\mathrm U}$ estimates how often two distinct human respondents give the same answer; $C(q)$ is the probability that two independent model answers match, using probabilities normalized over the listed options. Each panel shows the median collision ratio across evaluated question--group pairs for early-stopped adapters (blue line), with shaded bands spanning the $25$th--$75$th and $5$th--$95$th percentiles across pairs. The marker at $n=0$ is the zero-shot base model, with vertical whiskers repeating the two bands; the horizontal axis is symmetric-log so that $n=0$ is shown on the same axis as the sweep, and each panel has its own y-axis limits. }
\label{fig:exp4fullmatrix}
\end{figure}

Respondents are partitioned into demographic groups based on characteristics such as country, age, sex, race/urban-rural status, region, and education. For each fixed question and demographic group $X$, let $Y$ denote the answer and let $m_a$ of the $m$ respondents choose answer $a$. We estimate human collision probability as the proportion of pairs of distinct respondents who give the same answer: $\hat C_{\mathrm U}:=\sum_a m_a(m_a-1)/[m(m-1)]$. The prompt presented to the model includes the demographic group as a persona, the survey question, and its answer options, and asks the model to select a single answer option. Applying softmax over the answer-option logits gives $q(a\mid X)$, the model's probability of answer $a$ for the fixed question and demographic group $X$. The model's collision probability is $C(q):=\sum_a q(a\mid X)^2$. We report $R_{\mathrm U}:=C(q)/\hat C_{\mathrm U}$, which estimates the collision ratio in Eq.~\eqref{eq:relative-diversity} using $\hat C_{\mathrm U}$ as the human collision estimate.

We repeat this experiment using $376$ GSS questions \citep{davern2025}, $53$ WVS questions \citep{haerpfer2020}, and $44$ ANES questions \citep{anes2021}---and four models---\texttt{gemma-2-2b-it} \citep{gemma_2}, \texttt{Qwen3.5-2B} and \texttt{Qwen3.5-4B} \citep{qwen}, and \texttt{gemma-3-4b-it} \citep{gemma_3}---using the split defined in Table~\ref{tab:exp4counts} of Appendix \ref{app:survey-data-prompts}. For each survey, model, and fine-tuning set of $n$ respondent--question examples, we evaluate $R_{\mathrm U}$ for every question--group pair with at least $60$ respondents. We include the zero-shot model as $n=0$. This gives a fixed set of $18{,}270$ pairs on the GSS, $8{,}015$ on the WVS, and $8{,}211$ on the ANES. Figure~\ref{fig:exp4fullmatrix} plots, at each $n$, median pair-level ratios for early-stopped adapters, with shaded bands for the $25$--$75\%$ and $5$--$95\%$ spread across pairs. Movement of the median ratio toward $1$ means that model answer diversity is approaching the human level for a typical question--group pair.

Three of the four base models mode-collapse across all surveys; the fourth starts closest to median calibration. Zero-shot, \texttt{gemma-2-2b-it}, \texttt{Qwen3.5-4B}, and \texttt{gemma-3-4b-it} have median collision ratios $R_{\mathrm U}$ between $1.61$ and $2.92$, corresponding to roughly one-third to two-thirds of the human effective diversity. The collapse is most pronounced on the WVS---the sparsest and only cross-national survey. The three collapsed models all have higher median collision ratios on the WVS than on the U.S. surveys (Table~\ref{tab:exp4fullmatrix} in Appendix~\ref{app:survey-results}). The severity of mode collapse therefore varies with both the model and the survey population.

At the smallest $n$, the three collapsed models stay near their zero-shot collision ratio. \texttt{Qwen3.5-2B}, the model closest to median calibration at $n=0$, is instead driven above its zero-shot collision ratio during early fine-tuning on all three surveys---to $2.22$ on the GSS, leaving less than half its starting effective diversity---before recovering (Table~\ref{tab:exp4fullmatrix} in Appendix~\ref{app:survey-results}). Beyond roughly $3{,}000$ fine-tuning examples, all twelve survey--model combinations maintain median collision ratios within $10\%$ of the calibrated value of $1$. \texttt{Qwen3.5-2B}'s early increase in collision ratio moves it away from $1$, showing that small fine-tuning datasets can worsen diversity calibration. The later approach to $1$ across all models and surveys shows that larger human datasets improve calibration.

We additionally report distribution-level fit, respondent-held-out collision calibration, and option-conditioned KL diagnostics for the survey experiments in Appendix~\ref{app:survey-distributional-fit}.


\begin{figure}[tbp]
\centering
\includegraphics[width=\textwidth]{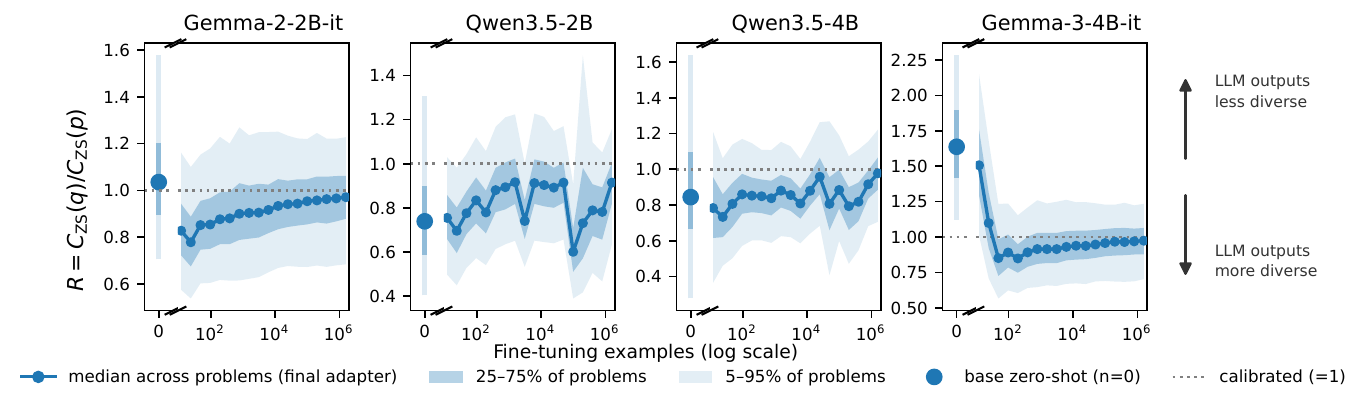}
\caption{\textbf{Experiment~3. CodeNet diversity under normalized Zhang--Shasha similarity.} Collision ratio $C_{\mathrm{ZS}}(q)/C_{\mathrm{ZS}}(p)$ against SFT examples ($1$ = calibrated). Lines show median finite collision ratios across problems; shaded bands span the $25$th--$75$th and $5$th--$95$th percentiles. At the full fine-tuning dataset, each model's median uses all $2{,}647$ candidate problems, while base markers cover $2{,}646$, $2{,}569$, $986$, and $2{,}646$ problems in panel order. The text separately compares the same problems before fine-tuning and after fine-tuning on the full dataset. Panels have separate y-axes and broken horizontal axes.}
\label{fig:codenet-full}
\end{figure}

\subsection{Experiment 3: code generation using CodeNet dataset}
\label{sec:exp-code}

In contrast to the previous section, exact sequence collisions are too rare to estimate for complete programs. So on CodeNet \citep{puri2021codenet} we compare model generations with accepted Python submissions using structural similarity. We fine-tune the same four LLMs as in Experiment~2 on accepted Python submissions. Before fine-tuning and at each fine-tuning-set size, each model generates $32$ programs per problem. We map each program $y$ to an ordered labeled canonical Python abstract syntax tree $T(y)$, removing comments and formatting and canonicalizing identifiers and literals. 

\begin{samepage}
Let $d_{\mathrm{ZS}}$ be the Zhang--Shasha tree edit distance \citep{zhangshasha1989}: the minimum number of node insertions, deletions, and relabelings needed to transform one tree into the other, with each operation costing one. We define
\begin{equation}
\textstyle k_{\mathrm{ZS}}(y,y'):=
1-\frac{d_{\mathrm{ZS}}\!\left(T(y),T(y')\right)}{|T(y)|+|T(y')|}.
\label{eq:codenetkernel}
\end{equation}
where $|T(y)|$ is the number of nodes in $T(y)$.
\end{samepage}
This similarity kernel $k_{\mathrm{ZS}}(y,y')$ lies in $[0,1]$ and equals $1$ for identical trees. For each problem, $p$ denotes the distribution of accepted human submissions and $q$ the model's output distribution. The kernel collisions $C_{\mathrm{ZS}}(p)$ and $C_{\mathrm{ZS}}(q)$ in Eq.~\eqref{eq:kernel-collision} measure expected similarity within each source using Eq.~\eqref{eq:codenetkernel}, conditional on both programs parsing as Python. For each source separately, we estimate its collision by averaging $k_{\mathrm{ZS}}(y,y')$ over up to $32$ distinct unordered pairs of parseable programs sampled without replacement. Figure~\ref{fig:codenet-full} reports the median finite ratio $C_{\mathrm{ZS}}(q)/C_{\mathrm{ZS}}(p)$ over a fixed panel of $2{,}647$ candidate problems. The base models' median collision ratios lie on both sides of the accepted-program target: both Qwen models are over-diverse, while both Gemma models are under-diverse. After fine-tuning on the full dataset, each model's median uses ratios from all candidate problems. To check whether changes in which problems enter the median explain the shift toward calibration, we compare each model before and after fine-tuning on the full dataset using only problems included at both stages (Appendix~\ref{app:codenet-data-prompts}). All four medians still move closer to $1$, so the improvement persists when the set of problems is held fixed. Appendix~\ref{app:codenet-data-prompts} gives the complete construction, prompt, fine-tuning configuration, and program examples.


\section{Conclusion}
\label{sec:discussion}
LLMs can exhibit either mode collapse or over-dispersion relative to a target population, even though recent literature has mostly emphasized mode-collapse. Finite-sample fitting has no built-in direction: variance and squared bias contribute to increasing the expected collision ratio, while target--bias alignment can offset them. The absolute difference between collision probabilities of the model and the target population is bounded by the square root of KL divergence, so approaching the target distribution forces diversity calibration. Across the synthetic, survey, and CodeNet experiments, sufficient fine-tuning data moves measured diversity toward the target from both under- and over-dispersed starting points. In the surveys, this improvement persists against respondents held out from both training and validation-loss evaluation. These results support evaluating diversity relative to a specified target alongside standard performance benchmarks, and motivate fine-tuning methods that use the diversity-calibrating properties of SFT. Obtaining sufficient coverage of the target distribution remains a challenge when data are limited.

\clearpage
\section*{AI use statement}

In this work, we used generative AI tools to: formulate mathematical claims,
provide critical ingredients for proving mathematical claims, assist in the
writing of proofs, support qualitative and thematic data analysis, implement
methods, polish writing. We take responsibility for the final content of this
work, including text, claims or artifacts produced with the aid of generative
AI.

\bibliographystyle{iclr2027_conference}
\bibliography{references}

@misc{anes2021,
  author       = {{American National Election Studies}},
  title        = {{ANES Time Series Cumulative Data File, 1948--2020 [dataset and documentation]}},
  year         = {2022},
  note         = {September 16, 2022 version},
  url          = {https://electionstudies.org/data-center/anes-time-series-cumulative-data-file/}
}

@inproceedings{anderson2024,
  author       = {Anderson, B. R. and Shah, J. H. and Kreminski, M.},
  title        = {{Homogenization effects of large language models on human creative ideation}},
  booktitle    = {ACM Creativity and Cognition},
  pages        = {413--425},
  year         = {2024},
  url          = {https://doi.org/10.1145/3635636.3656204}
}

@article{argyle2023,
  author       = {Argyle, L. P. and Busby, E. C. and Fulda, N. and Gubler, J. R. and Rytting, C. and Wingate, D.},
  title        = {{Out of One, Many: Using language models to simulate human samples}},
  journal      = {Political Analysis},
  volume       = {31},
  number       = {3},
  pages        = {337--351},
  year         = {2023},
  url          = {https://doi.org/10.1017/pan.2023.2}
}

@article{arnold2024phasetransitionsoutputdistribution,
  author       = {Arnold, J. and Holtorf, F. and Sch{\"a}fer, F. and L{\"o}rch, N.},
  title        = {{Phase transitions in the output distribution of large language models}},
  journal      = {arXiv preprint arXiv:2405.17088},
  year         = {2024},
  url          = {https://arxiv.org/abs/2405.17088}
}

@article{arora2017,
  author       = {Arora, S. and Zhang, Y.},
  title        = {{Do GANs actually learn the distribution? An empirical study}},
  journal      = {arXiv preprint arXiv:1706.08224},
  year         = {2017}
}

@inproceedings{arora2018,
  author       = {Arora, S. and Risteski, A. and Zhang, Y.},
  title        = {{Do GANs learn the distribution? Some theory and empirics}},
  booktitle    = {ICLR},
  year         = {2018}
}

@inproceedings{arora2022,
  author       = {Arora, A. and Meister, C. and Cotterell, R.},
  title        = {{Estimating the entropy of linguistic distributions}},
  booktitle    = {ACL},
  pages        = {175--195},
  year         = {2022},
  url          = {https://doi.org/10.18653/v1/2022.acl-short.20}
}

@article{banayeeanzade2026,
  author       = {Banayeeanzade, A. and Yang, Q. and Tarsadiya, D. and Bahrani, F. and Blas, L. and Samuel, A. and Jia, R. and Razaviyayn, M. and Karimireddy, S. P.},
  title        = {{Sampling More, Getting Less: Calibration is the diversity bottleneck in LLMs}},
  journal      = {arXiv preprint arXiv:2605.11128},
  year         = {2026}
}

@article{bisbee2024,
  author       = {Bisbee, J. and Clinton, J. D. and Dorff, C. and Kenkel, B. and Larson, J. M.},
  title        = {{Synthetic replacements for human survey data? The perils of large language models}},
  journal      = {Political Analysis},
  volume       = {32},
  number       = {4},
  pages        = {401--416},
  year         = {2024},
  url          = {https://doi.org/10.1017/pan.2024.5}
}

@inproceedings{chakrabarty2024,
  author       = {Chakrabarty, T. and Padmakumar, V. and Brahman, F. and Muresan, S.},
  title        = {{Creativity support in the age of large language models: An empirical study involving professional writers}},
  booktitle    = {ACM Creativity and Cognition},
  pages        = {132--155},
  year         = {2024},
  url          = {https://doi.org/10.1145/3635636.3656201}
}

@inproceedings{chakraborty2024,
  author       = {Chakraborty, S. and Qiu, J. and Yuan, H. and Koppel, A. and Manocha, D. and Huang, F. and Bedi, A. S. and Wang, M.},
  title        = {{MaxMin-RLHF: Alignment with diverse human preferences}},
  booktitle    = {ICML},
  series       = {PMLR},
  volume       = {235},
  pages        = {6116--6135},
  year         = {2024}
}

@article{chen2026,
  author       = {Chen, Y. and Liu, Y. and Meng, F.},
  title        = {{SED-SFT: Selectively encouraging diversity in supervised fine-tuning}},
  journal      = {arXiv preprint arXiv:2602.07464},
  year         = {2026}
}

@article{chung2025,
  author       = {Chung, J. J. Y. and Padmakumar, V. and Roemmele, M. and Sun, Y. and Kreminski, M.},
  title        = {{Modifying large language model post-training for diverse creative writing}},
  journal      = {arXiv preprint arXiv:2503.17126},
  year         = {2025}
}

@misc{davern2025,
  author       = {Davern, M. and Bautista, R. and Freese, J. and Herd, P. and Morgan, S. L.},
  title        = {{General Social Survey 1972--2024 [machine-readable data file]}},
  howpublished = {NORC ed. Chicago: NORC at the University of Chicago [producer and distributor]},
  year         = {2025},
  note         = {Data accessed from the GSS Data Explorer website},
  url          = {https://gssdataexplorer.norc.org}
}

@article{doshi2024,
  author       = {Doshi, A. R. and Hauser, O. P.},
  title        = {{Generative AI enhances individual creativity but reduces the collective diversity of novel content}},
  journal      = {Science Advances},
  volume       = {10},
  number       = {28},
  eid          = {eadn5290},
  year         = {2024},
  url          = {https://doi.org/10.1126/sciadv.adn5290}
}

@inproceedings{finlayson2024,
  author       = {Finlayson, M. and Hewitt, J. and Koller, A. and Swayamdipta, S. and Sabharwal, A.},
  title        = {{Closing the curious case of neural text degeneration}},
  booktitle    = {ICLR},
  year         = {2024}
}

@article{friedman2022,
  author       = {Friedman, D. and Dieng, A. B.},
  title        = {{The Vendi Score: A diversity evaluation metric for machine learning}},
  journal      = {Transactions on Machine Learning Research},
  year         = {2023}
}

@misc{girotra2023,
  author       = {Terwiesch, C. and Meincke, L. and Girotra, K. and Mollick, E. R. and Nave, G. and Ulrich, K. T.},
  title        = {{AI and its impact on creativity and diversity: An empirical study of LLM-generated product ideas}},
  howpublished = {SSRN 4526071},
  year         = {2024},
  url          = {https://doi.org/10.2139/ssrn.4526071}
}

@inproceedings{gxchen2025,
  author       = {GX-Chen, A. and Prakash, J. and Guo, J. and Fergus, R. and Ranganath, R.},
  title        = {{KL-Regularized Reinforcement Learning for Generative Modelling is Designed to Mode Collapse}},
  booktitle    = {ICLR},
  year         = {2026}
}

@book{haerpfer2020,
  editor       = {Haerpfer, C. and Inglehart, R. and Moreno, A. and Welzel, C. and Kizilova, K. and Diez-Medrano, J. and Lagos, M. and Norris, P. and Ponarin, E. and Puranen, B.},
  title        = {{World Values Survey: Round Seven---Country-Pooled Datafile Version 6.0.0}},
  publisher    = {JD Systems Institute and WVSA Secretariat},
  address      = {Madrid, Spain and Vienna, Austria},
  year         = {2022},
  url          = {https://doi.org/10.14281/18241.24}
}

@inproceedings{holtzman2019,
  author       = {Holtzman, A. and Buys, J. and Du, L. and Forbes, M. and Choi, Y.},
  title        = {{The curious case of neural text degeneration}},
  booktitle    = {ICLR},
  year         = {2020}
}

@article{jiao2015,
  author       = {Jiao, J. and Venkat, K. and Han, Y. and Weissman, T.},
  title        = {{Minimax estimation of functionals of discrete distributions}},
  journal      = {IEEE Transactions on Information Theory},
  volume       = {61},
  number       = {5},
  pages        = {2835--2885},
  year         = {2015},
  url          = {https://doi.org/10.1109/TIT.2015.2412945}
}

@article{karouzos2026,
  author       = {Karouzos, C. and Tan, X. and Aletras, N.},
  title        = {{Where does output diversity collapse in post-training?}},
  journal      = {arXiv preprint arXiv:2604.16027},
  year         = {2026}
}

@inproceedings{kirk2024,
  author       = {Kirk, R. and Mediratta, I. and Nalmpantis, C. and Luketina, J. and Hambro, E. and Grefenstette, E. and Raileanu, R.},
  title        = {{Understanding the effects of RLHF on LLM generalisation and diversity}},
  booktitle    = {ICLR},
  year         = {2024}
}

@article{klypa2026,
  author       = {Klypa, R. and Cherednichenko, O.},
  title        = {{Diversity in large language models under supervised fine-tuning}},
  journal      = {arXiv preprint arXiv:2605.00195},
  year         = {2026}
}

@inproceedings{li2025,
  author       = {Li, Z. and Chen, C. and Xu, T. and Qin, Z. and Xiao, J. and Luo, Z.-Q. and Sun, R.},
  title        = {{Preserving diversity in supervised fine-tuning of large language models}},
  booktitle    = {ICLR},
  year         = {2025}
}

@article{orlitsky2016,
  author       = {Orlitsky, A. and Suresh, A. T. and Wu, Y.},
  title        = {{Optimal prediction of the number of unseen species}},
  journal      = {PNAS},
  volume       = {113},
  number       = {47},
  pages        = {13283--13288},
  year         = {2016},
  url          = {https://doi.org/10.1073/pnas.1607774113}
}

@inproceedings{padmakumar2023,
  author       = {Padmakumar, V. and He, H.},
  title        = {{Does writing with language models reduce content diversity?}},
  booktitle    = {ICLR},
  year         = {2024},
  note         = {arXiv:2309.05196}
}

@article{paninski2003,
  author       = {Paninski, L.},
  title        = {{Estimation of entropy and mutual information}},
  journal      = {Neural Computation},
  volume       = {15},
  number       = {6},
  pages        = {1191--1253},
  year         = {2003},
  url          = {https://doi.org/10.1162/089976603321780272}
}

@book{pinsker1964,
  author       = {Pinsker, M. S.},
  title        = {{Information and Information Stability of Random Variables and Processes}},
  publisher    = {Holden-Day},
  year         = {1964}
}

@inproceedings{puri2021codenet,
  author       = {Puri, R. and Kung, D. S. and Janssen, G. and Zhang, W. and Domeniconi, G. and Zolotov, V. and Dolby, J. T. and Chen, J. and Choudhury, M. and Decker, L. and Thost, V. and Buratti, L. and Pujar, S. and Ramji, S. and Finkler, U. and Malaika, S. and Reiss, F.},
  title        = {{CodeNet: A large-scale AI for code dataset for learning a diversity of coding tasks}},
  booktitle    = {NeurIPS Datasets and Benchmarks},
  year         = {2021}
}

@inproceedings{rafailov2023,
  author       = {Rafailov, R. and Sharma, A. and Mitchell, E. and Manning, C. D. and Ermon, S. and Finn, C.},
  title        = {{Direct Preference Optimization: Your language model is secretly a reward model}},
  booktitle    = {NeurIPS},
  volume       = {36},
  pages        = {53728--53741},
  year         = {2023}
}

@article{seddik2024,
  author       = {Seddik, M. E. A. and Chen, S.-W. and Hayou, S. and Youssef, P. and Debbah, M.},
  title        = {{How bad is training on synthetic data? A statistical analysis of language model collapse}},
  journal      = {arXiv preprint arXiv:2404.05090},
  year         = {2024}
}

@article{shumailov2024,
  author       = {Shumailov, I. and Shumaylov, Z. and Zhao, Y. and Papernot, N. and Anderson, R. and Gal, Y.},
  title        = {{AI models collapse when trained on recursively generated data}},
  journal      = {Nature},
  volume       = {631},
  pages        = {755--759},
  year         = {2024},
  url          = {https://doi.org/10.1038/s41586-024-07566-y}
}

@inproceedings{siththaranjan2024,
  author       = {Siththaranjan, A. and Laidlaw, C. and Hadfield-Menell, D.},
  title        = {{Distributional preference learning: Understanding and accounting for hidden context in RLHF}},
  booktitle    = {ICLR},
  year         = {2024},
  note         = {arXiv:2312.08358}
}

@inproceedings{slocum2025,
  author       = {Slocum, S. and Parker-Sartori, A. and Hadfield-Menell, D.},
  title        = {{Diverse preference learning for capabilities and alignment}},
  booktitle    = {ICLR},
  year         = {2025}
}

@article{sourati2025,
  author       = {Sourati, Z. and Ziabari, A. S. and Dehghani, M.},
  title        = {{The homogenizing effect of large language models on human expression and thought}},
  journal      = {arXiv preprint arXiv:2508.01491},
  year         = {2026}
}

@inproceedings{tevet2021,
  author       = {Tevet, G. and Berant, J.},
  title        = {{Evaluating the evaluation of diversity in natural language generation}},
  booktitle    = {EACL},
  pages        = {326--346},
  year         = {2021},
  url          = {https://doi.org/10.18653/v1/2021.eacl-main.25}
}

@article{verbalized2025,
  author       = {Zhang, J. and Yu, S. and Chong, D. and Sicilia, A. and Tomz, M. R. and Manning, C. D. and Shi, W.},
  title        = {{Verbalized Sampling: How to Mitigate Mode Collapse and Unlock LLM Diversity}},
  journal      = {arXiv preprint arXiv:2510.01171},
  year         = {2025},
  url          = {https://arxiv.org/abs/2510.01171}
}

@article{zhangshasha1989,
  author       = {Zhang, K. and Shasha, D.},
  title        = {{Simple fast algorithms for the editing distance between trees and related problems}},
  journal      = {SIAM Journal on Computing},
  volume       = {18},
  number       = {6},
  pages        = {1245--1262},
  year         = {1989},
  url          = {https://doi.org/10.1137/0218082}
}

@article{wang2025,
  author       = {Wang, A. and Morgenstern, J. and Dickerson, J. P.},
  title        = {{Large language models that replace human participants can harmfully misportray and flatten identity groups}},
  journal      = {Nature Machine Intelligence},
  volume       = {7},
  number       = {3},
  pages        = {400--411},
  year         = {2025},
  url          = {https://doi.org/10.1038/s42256-025-00986-z}
}

@inproceedings{welleck2019,
  author       = {Welleck, S. and Kulikov, I. and Roller, S. and Dinan, E. and Cho, K. and Weston, J.},
  title        = {{Neural text generation with unlikelihood training}},
  booktitle    = {ICLR},
  year         = {2020}
}

@article{xie2025,
  author       = {Xie, Y. and Xie, Y.},
  title        = {{Variance reduction in output from generative AI}},
  journal      = {arXiv preprint arXiv:2503.01033},
  year         = {2025}
}

@article{xie2026,
  author       = {Xie, Y. and Liang, L. and Li, S. and Lu, Y. and Xiao, Z. and Shi, M. and Huang, J. and Wang, M. and Xie, Y.},
  title        = {{Evaluating the statistical realism of LLM-generated social science data}},
  journal      = {PNAS},
  volume       = {123},
  number       = {19},
  eid          = {e2538145123},
  year         = {2026},
  url          = {https://doi.org/10.1073/pnas.2538145123}
}

@article{gemma_2,
    title={Gemma},
    url={https://www.kaggle.com/m/3301},
    DOI={10.34740/KAGGLE/M/3301},
    publisher={Kaggle},
    author={{Gemma Team}},
    year={2024}
}

@article{gemma_3,
    title={Gemma 3},
    url={https://goo.gle/Gemma3Report},
    publisher={Kaggle},
    author={{Gemma Team}},
    year={2025}
}

@misc{qwen,
    title  = {{Qwen3.5}: Towards Native Multimodal Agents},
    author = {{Qwen Team}},
    month  = {February},
    year   = {2026},
    url    = {https://qwen.ai/blog?id=qwen3.5}
}

\newpage

\long\def\exptwoappendix{%
\noindent {\bf How the data are generated.} In setting~(a), the ground-truth language is a time-homogeneous latent-factor $k$-gram process with exact closed-form conditionals, vocabulary $|V|=1024$ and Markov order $k=16$. The token $x_t$ at position $t$ in a sequence depends only on the trailing window $h_t:=x_{t-16:t-1}$, where $x_{<t}$ denotes the sequence prefix before step $t$. Its next-token probabilities are given by the tempered softmax of a frozen random linear map $\ell$ of the positionally pooled context embeddings:
\begin{equation}
p(x_t\mid x_{<t})
  = p(x_t\mid h_t)
  := \bigl[\softmax\!\bigl(\ell(h_t)/(\tau_{\mathrm{gen}}\,s)\bigr)\bigr]_{x_t},
\label{eq:synth-markov}
\end{equation}
\begin{equation}
\textstyle \ell(h)
  := \mathrm{pool}_w(E_h)\,W_1^\top W_2^\top W_{\mathrm{out}}^\top + b,
\qquad
\mathrm{pool}_w(E_h):=\sum_{j=1}^{16} w_j\,E_{h_j},
\label{eq:synth-logits}
\end{equation}
where $E\in\mathbb{R}^{1024\times128}$ is the token embedding table, $E_h$ collects the context-token embeddings $E_{h_j}$, and $w\in\mathbb{R}^{16}$ contains positional weights $w_j$. The matrices $W_1\in\mathbb{R}^{256\times128}$ and $W_2\in\mathbb{R}^{256\times256}$ are successive linear maps, $W_{\mathrm{out}}\in\mathbb{R}^{1024\times256}$ maps to token logits, and $b$ is the logit-bias vector. The generation temperature is $\tau_{\mathrm{gen}}=0.6$, and $s$ is the standard deviation of logits on a reference sample of 16-grams, so temperature is invariant to the overall scale of the random weights. Every ingredient is drawn at random once from the master seed $42$ and fixed, so the language is fully reproducible. There is no intermediate nonlinearity. Because $\ell$ is linear of rank at most $256$, a small GPT can represent the generating distribution $p$ exactly and any gap between model and truth would not be limited by model capacity. Sequences have length $64$: the first $k=16$ tokens are drawn uniformly at random, and each subsequent token is sampled by ancestral multinomial draws from $p(\cdot\mid h)$ on the trailing $16$-token window; the pre-training contexts are the $16$-token windows at every prediction position $16,\dots,63$. The pre-training sample size $N$ counts sampled sequences, each contributing $48$ supervised next-token targets (the data seed is derived from $42$ per $N$). Because the conditionals are closed-form, the exact target diversity $N_2(p_h)=1/\sum_a p(a\mid h)^2$ is known at every context; it is high and varied (median $177$, range $7$--$398$). At $k=16$ exact contexts essentially never repeat, so the count-based/tabular prediction of Appendix~\ref{sec:result2} is structurally undefined here.
}

\long\def\exptwoconfigappendix{%
\noindent {\bf Diffuse prior in setting~(a).} For each seed, we initialize the GPT using its default random initialization, without preliminary fitting. The starting distribution $q_{0,h}$ is its softmax output at context $h$ before pre-training on the synthetic language. It is generally nonuniform and varies with context and initialization seed.

\noindent {\bf What we measure.} We report the ensemble normalized gap
$G_h=\E[R]-1=\E[C_h(q_{\hth})]/C_h(p)-1$ and its three $C(p_h)$-normalized terms from
Eq.~\eqref{eq:ifbox}. Isolating them requires independently pre-trained models at each $N$. We pre-train $100$ seeds at
each of $100$ log-spaced sequence counts from $32$k to $1$M ($10{,}000$ GPTs; $6.135$M parameters) and evaluate
them on $2{,}000$ common contexts. Each seed draws its own data, initialization, and minibatch stream, so the
covariance captures their total across-run variation. The closed-form $p_h$ permits direct estimation of $b_h$
from the across-seed mean. All validation losses are below the uniform baseline $\log 1024\approx6.93$ nats.
}

\long\def\expthreeappendix{%
\noindent {\bf Setup.} We fine-tune a $27.97$M-parameter GPT ($d_{\text{model}}=256$, $4$ heads, $3$ layers,
$K=50{,}000$) on a heavy-tailed order-$16$ Zipf-plus-context language with exponent $1.07$. Each length-$17$
sequence contributes one prediction target. Before SFT, we select the $32$ most probable tokens in the target language's fixed token-frequency distribution and rescale their probabilities to sum to one. We then pre-train each GPT for $400$ steps on random contexts, using this same desired next-token distribution for every context.
The analyzed artifacts contain $100$ seeds at each of $100$ sequence counts from $500$ to $10^6$; every model then
undergoes $4{,}000$ SFT steps.
}

\long\def\surveydataandpromptsappendix{%
\noindent {\bf Datasets.} We study three long-running probability surveys of human attitudes. The General Social Survey (GSS) is a nationally representative survey of adults in the United States, conducted since $1972$; we use the $1972$--$2024$ cumulative file ($74{,}485$ respondents with complete demographics). The World Values Survey (WVS) is a cross-national survey of social, political, religious, and economic values; we use wave~7 (fieldwork $2017$--$2023$). For computational tractability we restrict the WVS to a random one-third of its countries and a random one-third of its measurable questions, both drawn with a fixed seed ($42$) and frozen before any distribution is computed; the resulting subset comprises $32{,}227$ respondents with complete demographics across $22$ countries spanning every inhabited continent (e.g.\ India, Pakistan, Kenya, Morocco, Serbia, South Korea, Canada, the Netherlands). The American National Election Studies (ANES) is a long-running U.S.\ political survey; we use the Time-Series Cumulative Data File ($1948$--$2020$; $69{,}784$ respondents with complete demographics). From each survey we take every suitable categorical opinion item---substantive attitude/value/policy questions with a small fixed answer set, with value labels transcribed verbatim from each survey's codebook (no guessed labels)---excluding identifiers, weights, dates, interviewer/geographic metadata, raw numerics, multi-select ``mention'' batteries, and behavioral/frequency/factual items. For the GSS, items are identified by their attitudinal value-label sets (favor/oppose, agree--disagree, spending too-little/too-much, confidence, importance, likely, should, true/false). We then keep only the measurable items---those with at least one demographic group of $\ge 60$ respondents, the same threshold used for respondent-sample target estimation---leaving $376$ items for the GSS, $159$ for the WVS (of which the random-third restriction above retains $53$), and $44$ for the ANES; every retained item is evaluated. The seven curated items used in an earlier version of this experiment (capital punishment \texttt{cappun}, gun permits \texttt{gunlaw}, welfare spending \texttt{natfare}, happiness \texttt{happy}, financial satisfaction \texttt{satfin}, political views \texttt{polviews}, party identification \texttt{partyid} for the GSS, and the analogous WVS seven) are a strict subset of the GSS measurable set, which we exploit for the seven-vs-full comparison (GSS only: the WVS random-third restriction does not retain all seven curated WVS items).

\noindent {\bf Grouping, conditioning, and respondent-sample target.} Respondents are partitioned into demographic groups $X$ defined by a tuple of variables, brought to GSS parity across surveys:
\begin{itemize}
\item \textbf{GSS}: $X=(\text{decade},\ \text{age bucket},\ \text{sex},\ \text{race},\ \text{region},\ \text{degree})$;
\item \textbf{WVS}: $X=(\text{country},\ \text{survey year},\ \text{urban/rural},\ \text{age bucket},\ \text{sex},\ \text{education})$;
\item \textbf{ANES}: $X=(\text{decade},\ \text{age bucket},\ \text{sex},\ \text{race},\ \text{region},\ \text{education})$,
\end{itemize}
Age was bucketed into groups $\{18\text{--}29,30\text{--}44,45\text{--}59,60+\}$ in the GSS and $\{16\text{--}29,30\text{--}44,45\text{--}59,60+\}$ in the WVS and ANES. ANES and WVS aggregate calendar years (to decade and fieldwork year, respectively) to avoid sparse groups. For each fixed question and group $X$, with $Y$ denoting the answer, observed counts $m_a$ among $m$ respondents define the unweighted plug-in conditional $\hat p_a=m_a/m$ and plug-in collision $\hat C_{\mathrm{plug}}=\sum_a\hat p_a^2$. We report instead the distinct-pair U-statistic
$\hat C_{\mathrm U}=\sum_a m_a(m_a-1)/[m(m-1)]=(m\hat C_{\mathrm{plug}}-1)/(m-1)$.
The target uses all eligible respondents, including fine-tuning and validation records. It represents this unweighted sample, not a held-out or survey-weighted estimate of the wider population. We compute each pair's collision ratio $R_{\mathrm U}=C(q_{\mathrm{opt}})/\hat C_{\mathrm U}$ before taking medians across pairs. Evaluation requires at least $60$ respondents for a question--group pair. Table~\ref{tab:exp4counts} gives the sample sizes.

\begin{table}[H]
\centering
\caption{\textbf{Experiment~2 datasets (measurable-only).} Each survey is restricted to its measurable opinion questions (a question with $\ge 1$ group of $\ge 60$ respondents). ``Groups'' is the number of demographic tuples $X$; ``eval groups'' counts groups holding $\ge 60$ respondents; ``eval pairs'' counts evaluated $(\text{question},X)$ pairs. ``Examples'' are pooled (respondent, question) pairs, split at the respondent level; the $n$-grid is anchored to each survey's full pool $N$ and halves down to the smallest $n\ge 8$, so each survey has its own grid.}
\label{tab:exp4counts}
\smallskip
\footnotesize
\begin{tabular}{@{}l r c r r r r@{}}
\toprule
Survey & Respondents & Countries & Groups & Eval groups ($\ge60$) & Eval pairs & Questions \\
\midrule
GSS  & $74{,}485$ & --- (US) & $2{,}512$ & $345$ & $18{,}270$ & $376$ \\
WVS  & $32{,}227$ & $22$      & $956$ & $293$ & $8{,}015$ & $53$ \\
ANES & $69{,}784$ & --- (US) & $2{,}942$ & $352$ & $8{,}211$  & $44$ \\
\midrule
\multicolumn{7}{@{}p{0.95\textwidth}}{\footnotesize Examples (fine-tuning / validation), pooled (respondent, question) pairs: GSS $5{,}224{,}403$ ($4{,}706{,}940$ / $517{,}463$); WVS $1{,}620{,}333$ ($1{,}462{,}645$ / $157{,}688$); ANES $1{,}846{,}292$ ($1{,}662{,}788$ / $183{,}504$). Minimum-eval threshold $=60$ respondents per $(\text{question},X)$ pair. The $n$-grid is anchored per survey to the full fine-tuning pool $N$ and halves down to the smallest $n\ge 8$: GSS $20$ points $[4{,}706{,}940,\dots,8]$; WVS $18$ points $[1{,}462{,}645,\dots,11]$; ANES $18$ points $[1{,}662{,}788,\dots,12]$ (roughly six orders of magnitude on the GSS).} \\
\bottomrule
\end{tabular}
\end{table}

\noindent {\bf Prompt and measurement.} The prompt presents the group as a first-person persona followed by the question and lettered options, then requests one option letter; Tables~\ref{tab:gssprompt},~\ref{tab:wvsprompt}, and~\ref{tab:anesprompt} show the templates. At the generation position, we take the first token identifier of each letter's tokenizer encoding, assert that the candidate identifiers are distinct, and apply a softmax only across their logits. The resulting $q_{\mathrm{opt}}(Y\mid X)$ is exact conditional on the candidate set and has no sampling noise, but it is not the full next-token distribution. We fine-tune four models of at most $4$B parameters---\texttt{gemma-2-2b-it}, \texttt{Qwen3.5-2B}, \texttt{Qwen3.5-4B}, and \texttt{gemma-3-4b-it}---with shared hyperparameters and model-specific LoRA target scopes. A respondent-level split prevents leakage between the fine-tuning and validation splits, while the empirical target intentionally uses both splits. The pipeline saves an early-stopped adapter and the final adapter; callback state is not persisted across requeues, so the former is not guaranteed to be the global minimum-validation-loss checkpoint. The sample-size grid halves each survey's fine-tuning pool $N$ down to the smallest $n\ge8$ (GSS: $20$ points, $4{,}706{,}940$ to $8$; WVS: $18$, $1{,}462{,}645$ to $11$; ANES: $18$, $1{,}662{,}788$ to $12$). For multimodal \texttt{gemma-3-4b-it}, LoRA and measurement use only the text decoder.

\begin{table}[H]
\centering
\caption{\textbf{GSS prompt format.} The persona template (top) is instantiated from a group's variables; the model then receives the question with its lettered (graded) options and must reply with one letter. Five curated items are shown with their graded options and an illustrative answer; the full measurable set of $376$ items is evaluated.}
\label{tab:gssprompt}
\smallskip
\fbox{\begin{minipage}{0.95\textwidth}\footnotesize\ttfamily
It is the \{decade\}. You are a \{age\} year-old \{race\} \{sex\} living in the \{region\} region of the United States. Your highest education credential is: \{degree\}.\\[2pt]
\{question\}\\[2pt]
Options:\ \ A) \{option A\}\ \ B) \{option B\}\ \ \dots\\[2pt]
Respond with only the single letter of the option that best matches your view.
\end{minipage}}

\smallskip
\footnotesize
\begin{tabular}{@{}p{2.6cm} p{8.2cm} c@{}}
\toprule
Item (\texttt{var}) & Graded answer options & Ex.\\
\midrule
Death penalty for murder (\texttt{cappun}) & A) favor;\ B) oppose & A \\
General happiness (\texttt{happy}) & A) very happy;\ B) pretty happy;\ C) not too happy & B \\
Financial satisfaction (\texttt{satfin}) & A) pretty well satisfied;\ B) more or less satisfied;\ C) not satisfied at all & B \\
Political views (\texttt{polviews}) & A) extremely liberal;\ B) liberal;\ C) slightly liberal;\ D) moderate;\ E) slightly conservative;\ F) conservative;\ G) extremely conservative & D \\
Party identification (\texttt{partyid}) & A) strong democrat;\ B) not very strong democrat;\ C) independent, close to democrat;\ D) independent;\ E) independent, close to republican;\ F) not very strong republican;\ G) strong republican;\ H) other party & A \\
\bottomrule
\end{tabular}
\end{table}

\begin{table}[H]
\centering
\caption{\textbf{WVS prompt format.} Same template as the GSS up to the persona variables (country / year / urban-rural / age / sex / education). Five curated items are shown with their graded options and an illustrative answer; a random third of the measurable items ($53$ of $159$, seed $42$) is evaluated.}
\label{tab:wvsprompt}
\smallskip
\fbox{\begin{minipage}{0.95\textwidth}\footnotesize\ttfamily
It is \{year\}. You live in \{urban/rural area\} of \{country\}. You are a \{age\} year-old \{sex\}. Your highest level of education is \{education\}.\\[2pt]
\{question\}\\[2pt]
Options:\ \ A) \{option A\}\ \ B) \{option B\}\ \ \dots\\[2pt]
Respond with only the single letter of the option that best matches your view.
\end{minipage}}

\smallskip
\footnotesize
\begin{tabular}{@{}p{2.6cm} p{8.2cm} c@{}}
\toprule
Item (\texttt{var}) & Graded answer options & Ex.\\
\midrule
Interpersonal trust (\texttt{trust}) & A) most people can be trusted;\ B) need to be very careful & B \\
Feeling of happiness (\texttt{happy}) & A) very happy;\ B) rather happy;\ C) not very happy;\ D) not at all happy & B \\
Confidence in government (\texttt{conf\_govt}) & A) a great deal;\ B) quite a lot;\ C) not very much;\ D) none at all & C \\
Justifiability of divorce (\texttt{divorce\_just}) & A) 1 (never justifiable);\ B) 2;\ \dots;\ I) 9;\ J) 10 (always justifiable) & C \\
Left--right self-placement (\texttt{left\_right}) & A) 1 (left);\ B) 2;\ \dots;\ I) 9;\ J) 10 (right) & E \\
\bottomrule
\end{tabular}
\end{table}

\begin{table}[H]
\centering
\caption{\textbf{ANES prompt format.} Same template up to the persona variables (decade / age / sex / race / region / education). Five illustrative items are shown with their graded options and an answer; the full measurable set of $44$ items is evaluated. Endpoint-anchored issue scales keep numbered interior points with labelled poles.}
\label{tab:anesprompt}
\smallskip
\fbox{\begin{minipage}{0.95\textwidth}\footnotesize\ttfamily
It is the \{decade\}s. You are a \{age\} year-old \{sex\} living in \{region\} of the United States. You identify as \{race\} and your highest level of education is \{education\}.\\[2pt]
\{question\}\\[2pt]
Options:\ \ A) \{option A\}\ \ B) \{option B\}\ \ \dots\\[2pt]
Respond with only the single letter of the option that best matches your view.
\end{minipage}}

\smallskip
\footnotesize
\begin{tabular}{@{}p{2.9cm} p{8.0cm} c@{}}
\toprule
Item (\texttt{var}) & Graded answer options & Ex.\\
\midrule
Party identification (\texttt{VCF0301}) & A) strong Democrat;\ B) weak Democrat;\ C) independent--Democrat;\ D) independent;\ E) independent--Republican;\ F) weak Republican;\ G) strong Republican & D \\
Liberal--conservative (\texttt{VCF0803}) & A) extremely liberal;\ \dots;\ D) moderate;\ \dots;\ G) extremely conservative & D \\
Government health insurance (\texttt{VCF0806}) & A) 1 (government insurance plan);\ B) 2;\ \dots;\ F) 6;\ G) 7 (private insurance plan) & D \\
Interest in elections (\texttt{VCF0310}) & A) not much interested;\ B) somewhat interested;\ C) very much interested & B \\
Presidential approval (\texttt{VCF0450}) & A) approve;\ B) disapprove & A \\
\bottomrule
\end{tabular}
\end{table}
}

\long\def\surveyconfigsappendix{%
\noindent {\bf Fine-tuning hyperparameters.} All four models use LoRA adapters with shared hyperparameters and model-specific target scope. Fine-tuning and validation use the same shifted mean cross-entropy over the encoded answer-letter token(s) followed by EOS; prompt and padding positions are masked. Runs use \texttt{bf16}, gradient checkpointing, and base \texttt{use\_cache} disabled. Tables~\ref{tab:exp4hparams} and~\ref{tab:exp4lora} give shared settings and per-model scope. The text models apply LoRA to seven attention and multilayer perceptron (MLP) projections in every decoder layer; \texttt{gemma-3-4b-it} applies the same projections only within its \texttt{language\_model}, leaving the vision tower fixed. The adapters modify $0.47\%$--$0.79\%$ of parameters ($10.9$--$29.8$ million).

\begin{table}[H]
\centering
\caption{\textbf{Experiment~2 shared fine-tuning hyperparameters.}}
\label{tab:exp4hparams}
\smallskip
\footnotesize
\begin{tabular}{@{}l l@{}}
\toprule
Hyperparameter & Value \\
\midrule
LoRA rank $r$ & $16$ \\
LoRA $\alpha$ & $32$ \\
LoRA dropout & $0.05$ \\
LoRA bias & none \\
Task type & causal language modeling \\
Base precision & \texttt{bf16} \\
Optimizer & AdamW (\texttt{adamw\_torch}) \\
Learning rate & $2\times10^{-4}$ \\
Learning-rate schedule & cosine decay \\
Warmup ratio & $0.03$ \\
Weight decay & $0$ \\
Max sequence length & $256$ tokens \\
Epochs & $3$ (no step cap) \\
Effective batch size & $16$ \\
Loss & cross-entropy on answer-letter encoding and EOS (prompt masked) \\
Validation subset & $4{,}000$ held-out pairs \\
Evaluations per run & $12$ (evenly spaced) \\
Checkpointing & every $500$ steps \\
Seed & $42$ \\
\bottomrule
\end{tabular}
\end{table}

\begin{table}[H]
\centering
\caption{\textbf{Experiment~2 per-model LoRA configuration.} The shared hyperparameters of Table~\ref{tab:exp4hparams} apply to all; the LoRA target scope differs. Target projections are \texttt{q\_proj}, \texttt{k\_proj}, \texttt{v\_proj}, \texttt{o\_proj}, \texttt{gate\_proj}, \texttt{up\_proj}, \texttt{down\_proj}; for \texttt{gemma-3-4b-it} these are matched only within the \texttt{language\_model} decoder (vision tower excluded). Trainable counts are from \texttt{get\_peft\_model}.}
\label{tab:exp4lora}
\smallskip
\footnotesize
\begin{tabular}{@{}l l r r@{}}
\toprule
Model & LoRA targets & Trainable params & \% of total \\
\midrule
\texttt{gemma-2-2b-it} & 7 proj., all decoder layers          & $20{,}766{,}720$ & $0.79\%$ \\
\texttt{Qwen3.5-2B}    & 7 proj., all decoder layers          & $10{,}911{,}744$ & $0.49\%$ \\
\texttt{Qwen3.5-4B}    & 7 proj., all decoder layers          & $21{,}233{,}664$ & $0.47\%$ \\
\texttt{gemma-3-4b-it} & 7 proj., \texttt{language\_model} only & $29{,}802{,}496$ & $0.69\%$ \\
\bottomrule
\end{tabular}
\end{table}
}

\long\def\codenetdatasetappendix{%
\noindent {\bf Dataset and estimand.}
Project CodeNet contains programming-contest problems, submitted programs, and metadata including language, user, timestamp, and judge status \citep{puri2021codenet}. We restrict the target to parseable, accepted Python submissions and compare implementation structure rather than lexical form or execution output. We measure expected similarity under the normalized Zhang--Shasha kernel in Equation~\eqref{eq:codenetkernel}, conditional on both programs parsing successfully. For each problem and distribution, we average the kernel over a deterministic seed-$42$ sample of up to $32$ distinct unordered pairs, using every pair when fewer are available.

The metadata contain $13{,}916{,}868$ submissions, including $3{,}286{,}314$ labeled Python and $1{,}796{,}563$ both Python and accepted, across $3{,}113$ problems. We impose no restriction on users, source length, or revisions per user; multiple accepted revisions remain separate observations. A submission enters the usable target only if its description and source are available and its Python AST parses successfully.

We split solutions deterministically $90$--$10$ at the submission level. A problem may appear in both partitions, but a submission identifier cannot. The $90\%$ partition supplies SFT data and the $10\%$ partition validation loss; the structural target draws up to $32$ parseable accepted solutions from both. Problems with one usable solution remain available for SFT, while target collision requires at least two.
}

\long\def\codenetconfigappendix{%
\noindent {\bf Fine-tuning configuration.}
We use \texttt{google/gemma-2-2b-it}, \texttt{Qwen/Qwen3.5-2B}, \texttt{Qwen/Qwen3.5-4B}, and
\texttt{google/gemma-3-4b-it}. For fine-tuning-pool size $N$, the grid successively halves nested prefixes of a
deterministic submission order down to $n\ge8$. Each model runs for three epochs with the survey experiment's LoRA
rank, scaling, dropout, learning rate, cosine schedule, $4{,}000$-row validation subset, $12$ scheduled validation
evaluations, and effective batch size $16$. Fine-tuning and validation use shifted mean cross-entropy over code plus
the tokenizer's assistant terminator, with the prompt masked. We report the final adapter after three
epochs; validation loss monitors fine-tuning but does not select the reported checkpoint, and loss magnitudes are not
compared across tokenizers.

\noindent {\bf Decoding and parse conditioning.}
For each problem and checkpoint, we draw $32$ samples with temperature $1.0$, top-$p=1.0$, and at most
$1{,}024$ new tokens. The generation call does not set top-$k$ or \texttt{enable\_thinking}, so loaded model and
chat-template defaults supply them; model revisions and the effective generation configuration were not logged.
The loaded Qwen templates are asymmetric: \texttt{Qwen3.5-2B} defaults to non-thinking output, whereas
\texttt{Qwen3.5-4B} defaults to thinking output. Decoding is therefore not matched across models. Seeds derive
deterministically from seed $42$, problem identifier, fine-tuning size, and checkpoint. Outputs that fail AST parsing are
recorded but excluded from collision; a ratio is finite only when at least two outputs parse. Thus cross-model
diversity and parseability comparisons are confounded by decoding defaults, and all Zhang--Shasha diversity results
are conditional on parsing.
}

\long\def\codenetexamplesappendix{%
\noindent {\bf Two accepted-solution examples.}
The following pairs illustrate the primary metric using accepted Python submissions. Similarities are recomputed with exactly the normalized Zhang--Shasha kernel in Equation~\eqref{eq:codenetkernel}. Each table has four cells in the requested order: source code A, canonical AST A, source code B, and canonical AST B. The AST cells are complete automatic exports from the metric representation. They contain every retained node and every ordered, field-typed child edge; nothing is manually pruned. Source locations and \texttt{Load}/\texttt{Store} context nodes are absent because the metric itself excludes them.

\noindent {\bf High structural similarity: divisor rule (\texttt{p03125}).}
The problem gives positive integers $A$ and $B$. If $A$ divides $B$, the program prints $A+B$; otherwise, it prints $B-A$. Both accepted submissions use the same conditional implementation. Their normalized Zhang--Shasha similarity is $0.970588$ (edit distance $2$ across two $34$-node trees).

\begin{landscape}
\begin{table}[H]
\centering
\caption{\textbf{High-similarity accepted pair for problem \texttt{p03125}.} The four cells contain the complete source and complete automatically exported canonical AST for submissions \texttt{s613176273} and \texttt{s053364634}. Node classes in the AST are highlighted in purple; Python keywords, comments, and strings use distinct syntax colors.}
\label{tab:codenet-high-full-ast}
\setlength{\tabcolsep}{2pt}
\renewcommand{\arraystretch}{1}
\begin{tabular}{|>{\raggedright\arraybackslash}p{0.18\linewidth}|>{\raggedright\arraybackslash}p{0.30\linewidth}|>{\raggedright\arraybackslash}p{0.18\linewidth}|>{\raggedright\arraybackslash}p{0.30\linewidth}|}
\hline
\begin{minipage}[t]{\linewidth}
\textbf{Code A: \texttt{s613176273}}\par\smallskip
\lstinputlisting[style=paperpython]{examples/high_a.py}
\end{minipage}
&
\begin{minipage}[t]{\linewidth}
\textbf{Complete canonical AST A}\par\smallskip
\lstinputlisting[style=paperast]{examples/high_a.ast}
\end{minipage}
&
\begin{minipage}[t]{\linewidth}
\textbf{Code B: \texttt{s053364634}}\par\smallskip
\lstinputlisting[style=paperpython]{examples/high_b.py}
\end{minipage}
&
\begin{minipage}[t]{\linewidth}
\textbf{Complete canonical AST B}\par\smallskip
\lstinputlisting[style=paperast]{examples/high_b.ast}
\end{minipage}
\\
\hline
\end{tabular}
\end{table}
\newpage
\begin{table}[H]
\centering
\caption{\textbf{Low-similarity accepted pair for problem \texttt{p03135}.} The problem asks how much time passes in World A when a student studies for $T$ hours in World B, where time passes $X$ times as fast. Both accepted submissions compute $T/X$. Solution A converts a split input list in a loop and formats the result; solution B uses nested generator loops, helper functions, and type annotations. Their normalized Zhang--Shasha similarity is $0.407407$ (edit distance $64$ across trees with $44$ and $64$ nodes). The four cells contain the source (comments removed) and the complete automatically exported canonical AST for submissions \texttt{s227737564} and \texttt{s089174134}; the AST difference is not a manually prepared summary.}
\label{tab:codenet-low-full-ast}
\setlength{\tabcolsep}{2pt}
\renewcommand{\arraystretch}{1}
\begin{tabular}{|>{\raggedright\arraybackslash}p{0.18\linewidth}|>{\raggedright\arraybackslash}p{0.30\linewidth}|>{\raggedright\arraybackslash}p{0.18\linewidth}|>{\raggedright\arraybackslash}p{0.30\linewidth}|}
\hline
\begin{minipage}[t]{\linewidth}
\textbf{Code A: \texttt{s227737564}}\par\smallskip
\lstinputlisting[style=paperpython,basicstyle=\ttfamily\fontsize{4.5}{4.9}\selectfont]{examples/low_a.py}
\end{minipage}
&
\begin{minipage}[t]{\linewidth}
\textbf{Complete canonical AST A}\par\smallskip
\lstinputlisting[style=paperast,basicstyle=\ttfamily\fontsize{4.5}{4.9}\selectfont]{examples/low_a.ast}
\end{minipage}
&
\begin{minipage}[t]{\linewidth}
\textbf{Code B: \texttt{s089174134}}\par\smallskip
\lstinputlisting[style=paperpython,basicstyle=\ttfamily\fontsize{4.5}{4.9}\selectfont]{examples/low_b.py}
\end{minipage}
&
\begin{minipage}[t]{\linewidth}
\textbf{Complete canonical AST B}\par\smallskip
\lstinputlisting[style=paperast,basicstyle=\ttfamily\fontsize{4.5}{4.9}\selectfont]{examples/low_b.ast}
\end{minipage}
\\
\hline
\end{tabular}
\end{table}
\end{landscape}
}

\appendix
\section{Proofs and alignment conditions}
\label{app:proofs}

\subsection{Kernel-collision stability}

\begin{proof}[Proof of Theorem~\ref{thm:kernelstability}]
Kernel collision is an expectation of $k$ under product measures $P=p\otimes p$ and $Q=q\otimes q$. Because $k\in[0,1]$, the variational characterization of total variation gives
\[
|C_k(p)-C_k(q)|=|\E_P[k]-\E_Q[k]|\le\TV(P,Q).
\]
Pinsker's inequality \citep{pinsker1964} and the product identity $\mathrm{KL}(p\otimes p\|q\otimes q)=2\mathrm{KL}(p\|q)$ yield
\[
\TV(P,Q)\le\sqrt{\tfrac12\mathrm{KL}(P\|Q)}
=\sqrt{\mathrm{KL}(p\|q)}.
\]
Clipping collision to $[0,1]$ and inverting positive collision gives the remaining statements.
\end{proof}

For a random fitted model $q_{\hth}$, the fixed-model bound also controls the expected kernel-collision gap:
\begin{equation}
\left|\E_{\hth}[C_k(q_{\hth})]-C_k(p)\right|
\le \E_{\hth}\!\left[\sqrt{\mathrm{KL}(p\|q_{\hth})}\right]
\le \sqrt{\E_{\hth}[\mathrm{KL}(p\|q_{\hth})]}.
\label{eq:randomkernelbound}
\end{equation}
The first inequality uses the triangle inequality and Theorem~\ref{thm:kernelstability}; the second uses Jensen's inequality.

\subsection{A sufficient condition for nonnegative alignment}
\label{sec:dualregime}

Only the target--bias alignment in Eq.~\eqref{eq:ifbox} can be negative. The following condition on the fitted model is sufficient: tokens more likely under the target have no smaller bias. A concentrated initialization does not by itself imply this condition.

\begin{assumption}[Target--bias comonotonicity]\label{assmp:missingmass}
For every pair $a,b\in V$,
\begin{equation}
\big(p_h(a)-p_h(b)\big)\big(b_h(a)-b_h(b)\big)\ge0,
\qquad
b_h=\E[q_{\hth,h}]-p_h.
\end{equation}
\end{assumption}

Chebyshev's sum inequality then gives nonnegative covariance under the uniform measure on $V$. Since $\sum_a b_h(a)=0$,
\begin{equation}
\label{eq:missingmass}
p_h^\top b_h
=|V|\,\Cov_{a\sim\mathrm{Unif}(V)}\!\big(p_h(a),b_h(a)\big)
\ge0.
\end{equation}
All three terms inside the brackets in Eq.~\eqref{eq:ifbox} are therefore nonnegative, and
\begin{equation}
\label{eq:guaranteedhomog}
\E[C_h(q_{\hth})]-C_h(p)\ge0.
\end{equation}
For a uniform target, the alignment is exactly zero because $\sum_a b_h(a)=0$. In the synthetic experiments,
alignment after fitting is negative in the diffuse setting and positive only over an intermediate $N$ range in the
mode-aligned setting. Positive aggregate alignment does not by itself establish the pairwise comonotonicity
assumption, and neither observation shows that initialization determines the later sign.

\subsection{Relative collision bounds and population cross-entropy}

Dividing Theorem~\ref{thm:kernelstability} by exact target collision gives a relative, exact-collision bound:

\begin{equation}
\left|\,\frac{C(q_h)}{C(p_h)}-1\,\right|
\;=\;\frac{\big|C(q_h)-C(p_h)\big|}{C(p_h)}
\;\overset{\eqref{eq:cebound}}{\le}\;\frac{\sqrt{\mathrm{KL}(p_h\,\|\,q_h)}}{C(p_h)}
\;=\; N_2^{\mathrm{true}}(h)\,\sqrt{\mathrm{KL}(p_h\,\|\,q_h)}.
\label{eq:ceboundlf}
\end{equation}

\begin{proposition}[Relative bound on the collision ratio]\label{prop:tightbound}
Fix a context $h$ with true conditional $p_h$ and model conditional $q_h$, write $C(r)=\|r\|_2^2$, $N_2(r)=1/C(r)$, $N_2^{\mathrm{true}}(h)=N_2(p_h)$. The local collision ratio is $R=C(q_h)/C(p_h)$, which equals $1$ when the model's diversity is calibrated and exceeds $1$ when the model is more repetitive than the data. With $\delta=q_h-p_h$ (so $\sum_a\delta_a=0$),
\begin{equation}
\boxed{\;
\left|\,\frac{C(q_h)}{C(p_h)}-1\,\right| \;\le\; 2\sqrt{N_2^{\mathrm{true}}(h)}\;\|q_h-p_h\|_2 \;+\; N_2^{\mathrm{true}}(h)\,\|q_h-p_h\|_2^2 ,
\;}
\label{eq:tightbound}
\end{equation}
A looser bound follows using the Pearson divergence $\chi^2(q_h\,\|\,p_h)=\sum_a \delta_a^2/p_h(a)\ge\|\delta\|_2^2$ (finite when $q_h\ll p_h$):
\begin{equation}
\boxed{\;
\left|\,\frac{C(q_h)}{C(p_h)}-1\,\right| \;\le\; 2\sqrt{N_2^{\mathrm{true}}(h)\,\chi^2(q_h\,\|\,p_h)} \;+\; N_2^{\mathrm{true}}(h)\,\chi^2(q_h\,\|\,p_h).
\;}
\label{eq:tightboundchi2}
\end{equation}
\end{proposition}

\begin{proof}
Since $q_h=p_h+\delta$, $C(q_h)-C(p_h)=2\langle p_h,\delta\rangle+\|\delta\|_2^2$, so $C(q_h)/C(p_h)-1=(2\langle p_h,\delta\rangle+\|\delta\|_2^2)/C(p_h)$. By Cauchy--Schwarz, $|\langle p_h,\delta\rangle|\le\|p_h\|_2\|\delta\|_2=\sqrt{C(p_h)}\|\delta\|_2$, so the cross term is at most $2\|\delta\|_2/\sqrt{C(p_h)}=2\sqrt{N_2^{\mathrm{true}}(h)}\|\delta\|_2$, while the quadratic term is $N_2^{\mathrm{true}}(h)\|\delta\|_2^2$; the triangle inequality gives~\eqref{eq:tightbound}. For~\eqref{eq:tightboundchi2}, absolute continuity confines all sums to the support of $p_h$, where $1/p_h(a)\ge1$ gives $\|\delta\|_2^2\le\chi^2(q_h\|p_h)$.
\end{proof}

The $\ell_2$ and $\chi^2$ forms exploit exact-collision structure. Theorem~\ref{thm:kernelstability} is more general and uses forward KL, which equals population cross-entropy above the target entropy:
\[
\mathrm{KL}(p_h\,\|\,q_h) = \E_{a\sim p_h}\!\left[-\log q_h(a)\right] - \E_{a\sim p_h}\!\left[-\log p_h(a)\right],
\]
where the first term is the population validation objective and the second is irreducible target entropy. An empirical validation loss estimates the first term only under the same completion objective and context weighting. In particular, the survey fine-tuning loss includes the end-of-sequence (EOS) token, whereas its reported diversity uses an option-conditioned answer-position distribution; aggregate survey validation loss is not a direct estimate of the KL governing that ratio.

Let $\mu$ be a specified probability distribution over token contexts, including its token-position weighting, and define $\overline{\mathrm{KL}}_\mu=\E_{h\sim\mu}[\mathrm{KL}(p_h\|q_h)]$. Cauchy--Schwarz applied to Eq.~\eqref{eq:ceboundlf} gives
\begin{equation}
\E_{h\sim\mu}\left|\,\frac{C(q_h)}{C(p_h)}-1\,\right|
\le
\sqrt{\E_{h\sim\mu}\!\big[N_2^{\mathrm{true}}(h)^2\big]\;\overline{\mathrm{KL}}_\mu}\,,
\label{eq:cebound_global}
\end{equation}
provided the displayed moments are finite. Thus average relative collision error is controlled by population excess cross-entropy under the same context measure.

\section{Synthetic alignment before and after fitting}
\label{app:initial-alignment}

\noindent {\bf Regime (b): a mode-aligned starting distribution.}
The second synthetic setting first pre-trains each random GPT by soft cross-entropy to a Zipf-weighted distribution over
the $32$ highest-probability tokens of the target's global Zipf component. Low entropy alone does not determine alignment:
Eq.~\eqref{eq:crosslambda} is positive only when
$p_h^\top\bar q_{0,h}>C(p_h)$. The construction is chosen to satisfy that initial inequality.

After SFT, the normalized alignment term $2p_h^\top b_h/C(p_h)$ is positive only for
$N=2.3$k--$21.5$k and peaks at $0.193$ near $N=4{,}300$
(Figure~\ref{fig:exp2decomp}(b)). Outside that interval its sign reverses. At $N=500$, the normalized variance
$\Tr\Cov(q_{\hth,h})/C(p_h)$ is approximately $3.9$, and the ensemble normalized collision gap
$G_h=\E[C_h(q_{\hth})]/C(p_h)-1$ is $3.53$. Thus the starting distribution affects the observed trajectory but does not fix the
sign after fine-tuning.

\noindent {\bf The initial sign is observable.}
Define $b_{0,h}=\bar q_{0,h}-p_h$. Table~\ref{tab:preft} evaluates the decomposition before fitting to target data, where it is
independent of sample size. Intervals are $95\%$ nonparametric bootstrap intervals from $2{,}000$ resamples of
the common evaluation contexts. The decomposition identity holds to machine precision
($|\text{derived}-\text{measured}|\le9\times10^{-16}$), but these initial terms do not identify their
counterparts after fitting.

\begin{table}[H]
\centering
\small
\begin{tabular}{lcc}
\toprule
Starting distribution $\bar q_0$ & Diffuse (regime~(a)) & Mode-aligned (regime~(b)) \\
\midrule
$|V|$ & $1{,}024$ & $50{,}000$ \\
\addlinespace
$\Tr\Cov(q_{\theta_0,h})/C(p_h)$ & $0.068$ & $1.2\times10^{-3}$ \\
 & \scriptsize$[0.067,0.069]$ & \scriptsize$[1.14,1.22]{\times}10^{-3}$ \\
\addlinespace
$\|b_{0,h}\|_2^2/C(p_h)$ & $0.826$ & $1.829$ \\
 & \scriptsize$[0.823,0.829]$ & \scriptsize$[1.720,1.941]$ \\
\addlinespace
$2p_h^\top b_{0,h}/C(p_h)$ & $-1.651$ & $\mathbf{+2.413}$ \\
 & \scriptsize$[-1.656,-1.645]$ & \scriptsize$[+2.345,+2.479]$ \\
\addlinespace
$\frac{\E[C(q_{\theta_0,h})]}{C(p_h)}-1$ & $-0.756$ & $+4.243$ \\
 & \scriptsize$[-0.760,-0.752]$ & \scriptsize$[4.068,4.419]$ \\
\addlinespace
$p_h^\top\bar q_{0,h}$ & $9.8\times10^{-4}$ & $\mathbf{5.0\times10^{-2}}$ \\
 & \scriptsize$[9.757,9.761]{\times}10^{-4}$ & \scriptsize$[4.93,5.11]{\times}10^{-2}$ \\
\addlinespace
$C(p_h)$ & $6.9\times10^{-3}$ & $2.4\times10^{-2}$ \\
 & \scriptsize$[6.7,7.2]{\times}10^{-3}$ & \scriptsize$[2.28,2.43]{\times}10^{-2}$ \\
\bottomrule
\end{tabular}
\caption{\textbf{Initialization realizes opposite alignment signs.} The four rows from variance through normalized
gap form the $C(p_h)$-normalized decomposition for each $100$-seed starting ensemble; the final two rows are unnormalized. Each quantity is averaged over contexts after any indicated normalization.
Brackets give $95\%$ bootstrap intervals over common evaluation contexts. The mode-aligned start satisfies
$p_h^\top\bar q_{0,h}>C(p_h)$, whereas the diffuse start does not. This is an initial-condition comparison, not a
claim that initialization fixes alignment after fitting.}
\label{tab:preft}
\end{table}

\section{Complete survey results}
\label{app:survey-results}

\begin{table}[H]
\centering
\caption{\textbf{Final-adapter survey medians.} Values are
$R_{\mathrm U}=C(q_{\mathrm{opt}})/\hat C_{\mathrm U}$ across evaluated question--group pairs, where $\hat C_{\mathrm U}$ is the distinct-pair U-statistic.
``Base'' is the zero-shot checkpoint; the other columns summarize final adapters over the SFT data-size grid. At the largest pool,
all medians round to $0.99$--$1.02$.}
\label{tab:exp4fullmatrix}
\smallskip
\footnotesize
\begin{tabular}{@{}l l r r r@{}}
\toprule
Survey & Model & Base ($n=0$) & Max over grid (final) & Largest $n$ (final) \\
\midrule
GSS  & \texttt{gemma-2-2b-it} & $2.11$ & $1.95$ & $1.00$ \\
GSS  & \texttt{Qwen3.5-2B} & $0.99$ & $2.22$ & $0.99$ \\
GSS  & \texttt{Qwen3.5-4B} & $1.61$ & $1.18$ & $1.00$ \\
GSS  & \texttt{gemma-3-4b-it} & $2.28$ & $2.19$ & $1.01$ \\
\midrule
WVS  & \texttt{gemma-2-2b-it} & $2.42$ & $1.92$ & $1.00$ \\
WVS  & \texttt{Qwen3.5-2B} & $1.04$ & $1.39$ & $1.02$ \\
WVS  & \texttt{Qwen3.5-4B} & $2.02$ & $1.76$ & $1.01$ \\
WVS  & \texttt{gemma-3-4b-it} & $2.92$ & $2.72$ & $1.01$ \\
\midrule
ANES & \texttt{gemma-2-2b-it} & $2.05$ & $1.95$ & $0.99$ \\
ANES & \texttt{Qwen3.5-2B} & $1.10$ & $1.34$ & $1.01$ \\
ANES & \texttt{Qwen3.5-4B} & $1.77$ & $1.97$ & $1.01$ \\
ANES & \texttt{gemma-3-4b-it} & $2.53$ & $2.43$ & $1.01$ \\
\bottomrule
\end{tabular}
\end{table}

Among base checkpoints, \texttt{Qwen3.5-2B} is closest to median calibration
($R_{\mathrm U}=0.99$--$1.10$), while the other three are under-dispersed on every survey
($1.61$--$2.92$). These medians do not establish pairwise distributional fit or reveal the target--bias alignment
term, which would require replicated fine-tuning runs.

\clearpage
\section{Survey distributional fit and held-out calibration}
\label{app:survey-distributional-fit}

We use the final three-epoch adapters and training grids of Section~\ref{sec:exp-survey}, with $n$ respondent--question training examples and $n=0$ denoting the base model. All targets use unweighted human response counts, and pair-level metrics receive equal weight. We evaluate $18{,}270$ GSS, $7{,}722$ WVS, and $8{,}211$ ANES question--group pairs.

\subsection{Distribution-level fit and option-conditioned KL}
\label{app:survey-distances}
\label{app:survey-kl-diagnostic}

For each fixed question--demographic-group pair, $q_n(a)$ is the model probability of answer $a$, normalized over the listed options, and $\widehat p_S(a)=m_{S,a}/\sum_bm_{S,b}$ is the empirical distribution for respondent set $S$. The normalized stored model vector $\widetilde q_n$ contains probabilities rounded to six decimals. We compute
\begin{align}
\label{eq:survey-tv-js}
\TV(p,q)&=\tfrac12\sum_a|p(a)-q(a)|,\\
\mathrm{JS}(p,q)&=\tfrac12\mathrm{KL}(p\|r)+\tfrac12\mathrm{KL}(q\|r),
\qquad r=(p+q)/2.
\end{align}
To handle rounded zeros, the option-conditioned KL diagnostic uses
\begin{equation}
\label{eq:survey-kl-regularized}
\begin{aligned}
\widetilde q_n^{(\varepsilon)}(a)
&=\frac{\max\{\widetilde q_n(a),\varepsilon\}}
{\sum_b\max\{\widetilde q_n(b),\varepsilon\}},\\
K_{\varepsilon,S}(n)&=\mathrm{KL}\!\left(\widehat p_S\,\middle\|\,\widetilde q_n^{(\varepsilon)}\right).
\end{aligned}
\end{equation}
We use natural logarithms and $\varepsilon=10^{-7}$, denoted $\mathrm{KL}^{*}$, with sensitivity checks at $10^{-6}$ and $10^{-8}$.

Figure~\ref{fig:survey-distributional-fit} compares $\widetilde q_n$ with the full empirical human target, including training and validation respondents. From base to full fitting, median TV, JS, and regularized KL decrease in all $12$ model--survey combinations. Median TV ranges from $0.318$--$0.763$ at base to $0.057$--$0.109$ at full fitting; JS decreases from $0.0700$--$0.3811$ to $0.0030$--$0.0116$ nats. Median TV has Spearman correlation $-0.998$ to $-0.899$ with $\log n$ across positive training sizes. Median per-pair absolute collision gaps, $|C(q_n)-C(\widehat p)|$, decrease in all $12$ model--survey combinations, from a range of $0.074$--$0.630$ at base to $0.018$--$0.031$ at full-data fitting. At full-data fitting, the across-pair Spearman correlation between regularized KL and the absolute gap ranges from $0.26$ to $0.46$.

Median $K_{10^{-7}}$ decreases from $0.2850$--$8.2350$ to $0.0117$--$0.0454$ nats, and the decrease persists in every combination at all three floors. At full fitting, its across-pair Spearman correlation with $|C(q_n)/C(\widehat p)-1|$ is $0.481$--$0.599$, using plug-in empirical collision $C(\widehat p)=\sum_a\widehat p(a)^2$. These regularized option-conditioned measurements support improved distributional fit alongside diversity calibration, consistent with Theorem~\ref{thm:kernelstability}.

\begin{figure}[tp]
\centering
\includegraphics[width=\textwidth]{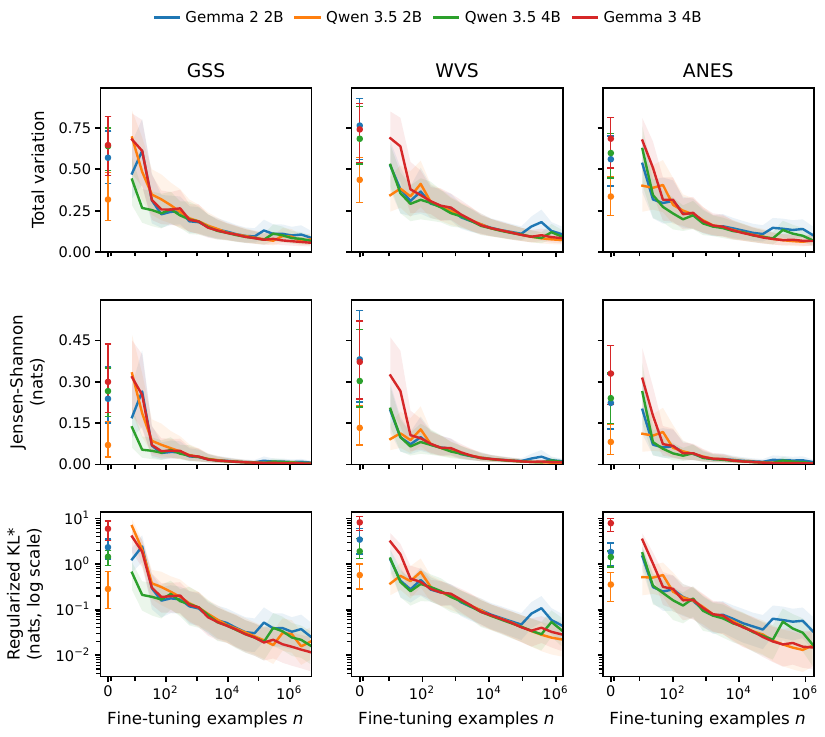}
\caption{\textbf{Survey distributional fit across fine-tuning sizes.} Columns identify surveys; rows show TV, JS, and regularized forward KL between each pair's full empirical human distribution, including training and validation respondents, and its model answer-option distribution. Colors identify models. Lines are medians across $18{,}270/7{,}722/8{,}211$ fixed question--group pairs in GSS/WVS/ANES; shading gives their $25$th--$75$th percentiles. Dots at $n=0$ show base models with the same percentile whiskers; positive sizes use final three-epoch adapters. Each pair has at least $60$ human answers. JS and KL use nats; KL uses the $10^{-7}$ floor in Eq.~\eqref{eq:survey-kl-regularized} and a logarithmic vertical axis. Horizontal axes are symmetric-log, linear through $n=8$.}
\label{fig:survey-distributional-fit}
\end{figure}

\subsection{Respondent-held-out calibration}
\label{app:survey-heldout}

We compare three human targets:
\begin{enumerate}[leftmargin=*,nosep]
\item \textbf{Full:} all eligible respondents, including training and validation respondents, with at least $60$ answers per pair.
\item \textbf{Validation:} the original $10\%$ respondent pool held out from gradient training, with at least $10$ answers per pair. Validation-loss samples are drawn from this pool.
\item \textbf{Validation-unused:} validation respondents absent from every sampled validation-loss pair, with at least $5$ answers per pair. We reconstruct the original $4{,}000$-pair sample using seed $42$ and all configured questions, then exclude each sampled respondent across all their answers.
\end{enumerate}
Pair sets are fixed across models and training sizes within each target. This retrospective restriction retains empirical respondent subsets; respondents answering more questions have a greater chance of entering the validation-loss sample.

For answer counts $m_a$ and total $m$, the distinct-respondent collision and model--human ratio are
\begin{equation}
\label{eq:survey-heldout-collision}
\widehat C_{\mathrm U}=\frac{\sum_am_a(m_a-1)}{m(m-1)},
\qquad
R_{\mathrm U}(n)=\frac{C(q_n)}{\widehat C_{\mathrm U}}.
\end{equation}
The numerator uses the recorded model collision before probability-vector rounding. Calibration corresponds to $R_{\mathrm U}=1$; values above $1$ indicate under-diversity and below $1$ over-dispersion. The U-statistic satisfies $\widehat C_{\mathrm U}=[mC(\widehat p)-1]/(m-1)$, correcting the self-pair contribution in plug-in empirical collision.

\begin{table}[H]
\centering
\caption{\textbf{Fixed survey target coverage.} Counts are pairs with positive human collision and a defined ratio. The validation-unused target additionally contains $26/40/22$ zero-collision pairs in GSS/WVS/ANES, retained for distributional distances.}
\label{tab:survey-heldout-coverage}
\smallskip
\footnotesize
\begin{tabular}{@{}l r r r r@{}}
\toprule
Target & Minimum answers & GSS & WVS & ANES\\
\midrule
Full & $60$ & $18{,}270$ & $7{,}722$ & $8{,}211$\\
Validation & $10$ & $9{,}051$ & $3{,}596$ & $3{,}836$\\
Validation-unused & $5$ & $11{,}691$ & $1{,}219$ & $4{,}230$\\
\bottomrule
\end{tabular}
\end{table}

Under both held-out targets, median per-pair $|R_{\mathrm U}-1|$ decreases from base to full fitting in all $12$ combinations (Figure~\ref{fig:survey-heldout-calibration}). Full-data median ratios are $1.008$--$1.052$ for validation respondents and $1.033$--$1.079$ for validation-unused respondents. In the latter set, median TV, JS, and regularized KL also decrease in every combination; median TV changes from $0.380$--$0.799$ to $0.187$--$0.282$. The unused WVS target retains $1{,}219$ finite-ratio pairs with a median of five answers per eligible pair. Shaded bands summarize variation across question--group pairs, with target sizes and coverage given in Table~\ref{tab:survey-heldout-coverage}.

\begin{figure}[tp]
\centering
\includegraphics[width=\textwidth]{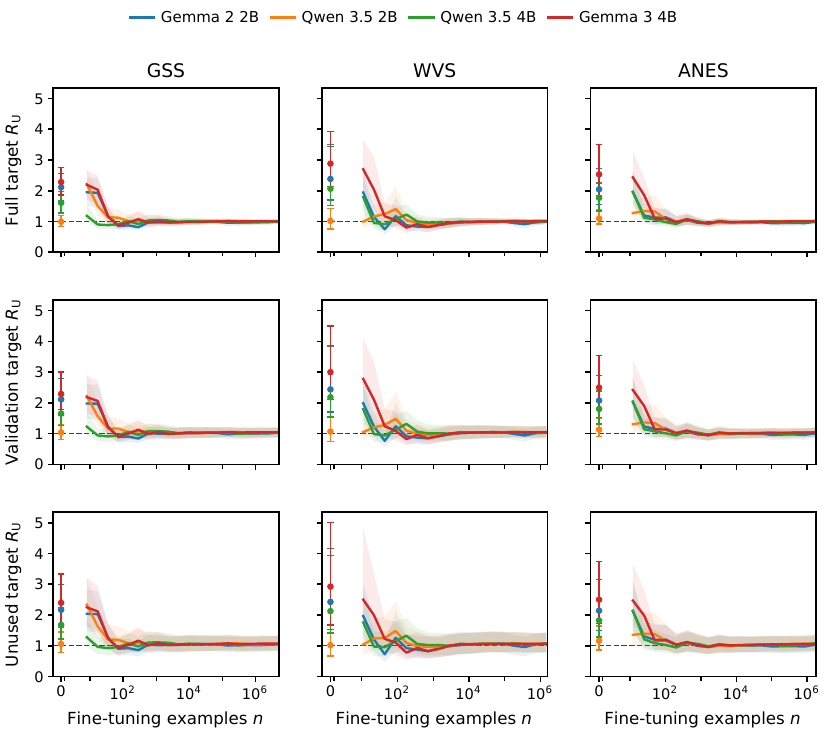}
\caption{\textbf{Survey collision calibration under full and respondent-held-out targets.} Columns identify surveys. Rows use all respondents, respondents held out from training, and respondents unused by both training and validation-loss evaluation, with minimum answer counts $60$, $10$, and $5$. The statistic $R_{\mathrm U}=C(q)/\widehat C_{\mathrm U}$ compares the probability that two independent option-conditioned model answers agree with the fraction of distinct human respondent pairs agreeing. The dashed line at $1$ denotes calibration. Colors identify models. Lines and bands show medians and $25$th--$75$th percentiles across fixed pairs for final three-epoch adapters; base-model dots carry the same percentile whiskers. Pair counts appear in Table~\ref{tab:survey-heldout-coverage}; pair sets differ between rows. Sample size $n$ counts respondent--question examples, with symmetric-log axes linear through $n=8$.}
\label{fig:survey-heldout-calibration}
\end{figure}
\clearpage

\section{Mitigation methods and finite-sample estimation}

\noindent {\bf Mitigation via decoding and fitting objectives.} Several methods improve text quality and diversity in maximum-likelihood models by modifying decoding or fitting. Decoding methods include Top-$k$ and nucleus sampling \citep{holtzman2019} and truncation sampling \citep{finlayson2024}; the unlikelihood objective instead changes the fitting criterion \citep{welleck2019}. Methods specific to SFT include GEM, using entropy-regularized distribution matching \citep{li2025}; SED-SFT, selectively regularizing entropy during exploration \citep{chen2026}; and TOFU, employing tempered focal loss to address the neglect of low-frequency patterns \citep{klypa2026}. Finally, fine-tuning objectives enhance diversity in creative writing \citep{chung2025}.  Further approaches preserve diversity during preference optimization by separating the entropy and reference-cross-entropy components of the standard KL penalty \citep{slocum2025}, or by preserving distributions or mixtures of preferences instead of compressing feedback to a single scalar value \citep{siththaranjan2024,chakraborty2024}. These methods provide benchmarks for our work, which contributes a finite-sample law quantifying the portion of observed diversity gaps attributable to estimation error.

\noindent {\bf Finite-sample estimation of distributional functionals.} Plug-in entropy estimators are susceptible to downward bias in finite samples, especially when the alphabet size is large relative to the sample size (Miller--Madow, Basharin; see \citealp{paninski2003}). The inverse-Simpson number, $N_2 = \exp H_2$, represents the order-2 Hill number.  Prior work characterizes minimax estimation of discrete functionals, including entropy and power sums \citep{jiao2015}, estimates entropy for linguistic distributions \citep{arora2022}, and investigates optimal prediction of unseen species, a task directly relevant to identifying valid, unobserved continuations \citep{orlitsky2016}. Finally, diversity diagnostics have been developed, including collision-based birthday-paradox tests for low effective support in generative adversarial networks (GANs) \citep{arora2017,arora2018} and the Vendi score, which measures similarity-aware diversity beyond exact labels via the Shannon entropy of a similarity matrix's eigenvalues \citep{friedman2022}. These results are adapted to the context of autoregressive LLM fine-tuning with shared neural parameters, variable prompts and context lengths, and positive-only supervision, to interpret findings in terms of scaling laws and prior-relative lower bounds.

\section{Sequence-level bounds and response length}
\label{sec:lengthscaling}

Theorem~\ref{thm:kernelstability} applies directly to the countable response space $V^*$. For a bounded sequence
kernel $k:V^*\times V^*\to[0,1]$, define
$C_{p,k}(x)=\E_{Y,Y'\overset{\mathrm{iid}}{\sim}p(\cdot\mid x)}[k(Y,Y')]$ and define
$C_{q,k}(x)$ analogously. Then
\begin{equation}
\label{eq:ceboundseq}
\big|\,C_{p,k}(x)-C_{q,k}(x)\,\big|
\;\le\;\sqrt{\mathrm{KL}\big(p(\cdot\mid x)\,\big\|\,q(\cdot\mid x)\big)}.
\end{equation}

\noindent {\bf Length enters through total excess cross-entropy.}
If $p$ and $q$ factor over the same tokenization and termination rule, the KL chain rule gives
\begin{equation}
\label{eq:klchain}
\mathrm{KL}\big(p(\cdot\mid x)\,\big\|\,q(\cdot\mid x)\big)
\;=\;\E_{Y\sim p(\cdot\mid x)}\!\left[\;\sum_{t=1}^{T(Y)+1}\mathrm{KL}\big(p_{h_t}\,\big\|\,q_{h_t}\big)\right],
\qquad h_t=(x,Y_{<t}),
\end{equation}
where $T(Y)+1$ includes EOS. Let $L(x)=\E_{Y\sim p}[T(Y)+1]$ and let
$\overline{\mathrm{KL}}_{\mathrm{tok}}(x)$ be the right-hand side divided by $L(x)$. Equation~\eqref{eq:ceboundseq}
becomes
\begin{equation}
\label{eq:ceboundlength}
\big|\,C_{p,k}(x)-C_{q,k}(x)\,\big|
\;\le\;\sqrt{L(x)\,\overline{\mathrm{KL}}_{\mathrm{tok}}(x)}\,.
\end{equation}

At fixed mean per-token KL, this upper bound grows as $\sqrt{L(x)}$. It improves on the trivial maximum gap of $1$ only when total sequence KL is below one nat. The survey statistic is a separate one-step, option-conditioned categorical comparison, whereas its
fine-tuning and validation losses cover the answer encoding and EOS over the full vocabulary; those aggregate losses
are not the KL in a bound for $q_{\mathrm{opt}}$. CodeNet responses can be hundreds of tokens long, and the paper
does not estimate their sequence-level KL, so Eq.~\eqref{eq:ceboundlength} is not quantitatively informative for the
reported Zhang--Shasha collision ratios.

\noindent {\bf The relative sequence bound is usually vacuous.}
Dividing the absolute exact-match bound by target collision incurs the factor $1/C_p(x)$, the target's effective
number of complete responses. For open-ended generation, $C_p(x)$ can be extremely small, so $1/C_p(x)$ is large
and exact collisions are difficult to estimate at feasible sample sizes. A bounded structural kernel can yield a
larger, more estimable $C_{p,k}(x)$, but the usefulness of its relative bound still depends on the target kernel
collision and sequence-level KL.

\section{KL-regularized preference optimization}
\label{sec:rl}

The accounting framework also distinguishes finite-sample error from a preference objective's population optimum.
For KL-regularized reward optimization, that optimum is the Gibbs tilt
$\pi^\star(y\mid x)\propto\pi_{\mathrm{ref}}(y\mid x)\exp\!\big(r(x,y)/\beta\big)$, where $r$ is the reward,
$\pi_{\mathrm{ref}}$ is the reference policy, and $\beta>0$ is the KL strength. Under its preference-model
assumptions, DPO parameterizes the same implicit optimum \citep{rafailov2023}. Whether that optimum is more
concentrated than a human-response target depends on the reward, reference, and target; we do not establish such
concentration here.
\begin{equation}
\label{eq:rlcollision}
C(\pi^\star)=\frac{\sum_y\pi_{\mathrm{ref}}(y\mid x)^2\exp\!\big(2r(x,y)/\beta\big)}{\Big[\sum_y\pi_{\mathrm{ref}}(y\mid x)\exp\!\big(r(x,y)/\beta\big)\Big]^2},
\end{equation}
For a random fitted policy $\hat\pi$, add and subtract the collision of $\pi^\star$:
\begin{equation}
\label{eq:rldecomp}
\boxed{\;
\E[C(\hat\pi)]-C(p)=
\underbrace{\big(\E[C(\hat\pi)]-C(\pi^\star)\big)}_{\text{estimation, approximation, and optimization}}
+\underbrace{\big(C(\pi^\star)-C(p)\big)}_{\text{objective-induced}}.
\;}
\end{equation}
The first term can vanish with growing data only under suitable consistency, capacity, and optimization conditions.
The second can remain even then because $\pi^\star$ need not equal $p$. Equation~\eqref{eq:rldecomp} is an
accounting identity, not an empirical result of this paper; residual error may reflect either term.

\section{Empirical check of the bound}
\label{sec:bound-check}
Proposition~\ref{prop:tightbound} provides a relative bound for fixed fitted-model conditionals. Across
$150{,}000$ held-in contexts from $100$ GPTs---one per pre-training sample size from $32$k to $1$M for a fixed
$6.14$M-parameter model---the measured $|C(q_h)/C(p_h)-1|$ never exceeds the $\ell_2$ bound. The median
measured-to-bound ratio is $0.140$ and the maximum is $0.980$: the bound is typically loose and occasionally
near-tight. The $\chi^2$ and forward-KL forms also hold everywhere but are looser. The minimum measured KL is
$0.0353>0$.

\begin{figure}[H]
\centering
\includegraphics[width=0.72\textwidth]{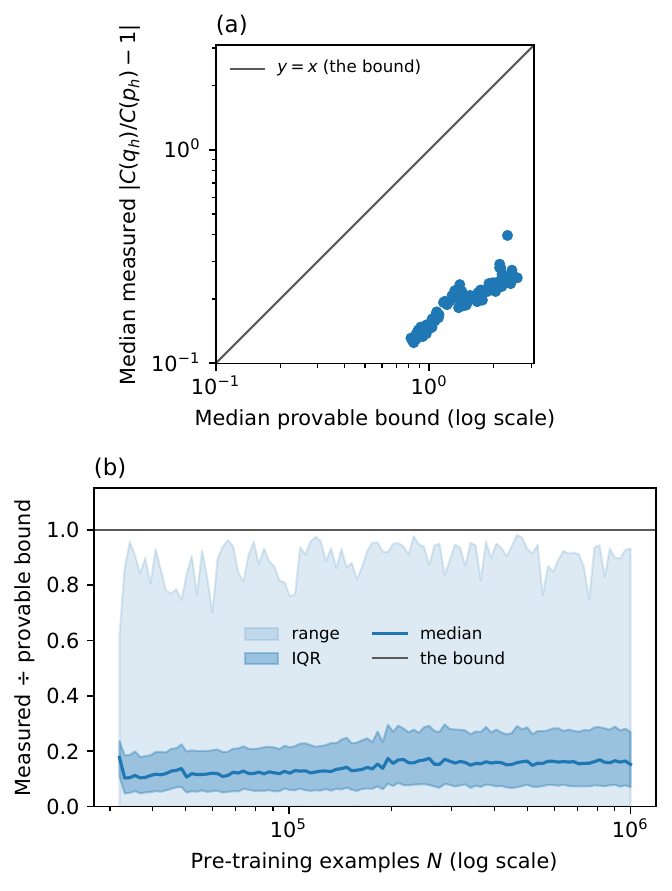}
\caption{\textbf{The $\ell_2$ bound holds everywhere and is occasionally near-tight.} Across $150{,}000$
held-in contexts, the median measured-to-bound ratio is $0.140$ and the maximum is $0.980$. (a) Median measured
error against median bound for each $N$; all points lie below $y=x$. (b) Median, interquartile band, and full range
of the ratio across $1{,}500$ contexts per $N$.}
\label{fig:exp2bound}
\end{figure}

Complete per-question performance grids for all four models are available in the code repository accompanying this submission. These grids detail performance on every measurable item.

\section{Experimental details}
\label{app:data-prompts}

\subsection{Synthetic languages}
\label{app:synthetic-data}
\exptwoappendix
\exptwoconfigappendix
\expthreeappendix

\subsection{Surveys}
\label{app:survey-data-prompts}
\surveydataandpromptsappendix
\surveyconfigsappendix

\subsection{CodeNet}
\label{app:codenet-data-prompts}
\codenetdatasetappendix
\codenetconfigappendix

\noindent {\bf Comparison on the same problems.}
For each model, we retain problems with at least two parseable outputs both before fine-tuning and after fine-tuning on the full dataset.
In Figure~\ref{fig:codenet-full} panel order, these sets contain $2{,}646$, $2{,}569$, $986$, and $2{,}646$ problems.
Their median collision ratios change from $1.035$ to $0.971$, $0.740$ to $0.915$, $0.844$ to $1.003$, and $1.637$ to $0.974$, respectively.
Every model's median therefore moves closer to $1$ when evaluated on the same problems at both stages.
Both estimates remain conditional on parsing.
On the fixed panel of $2{,}647$ problems, the share of sampled outputs that parse changes from $94.8\%$ to $91.3\%$, $43.1\%$ to $67.4\%$, $13.5\%$ to $89.7\%$, and $99.0\%$ to $91.5\%$ from base to full-data fine-tuning, in figure panel order.

\noindent {\bf Code-generation prompt.}
The problem description is inserted after the following instruction in one user turn; each model's chat template supplies its own control tokens.
\begin{quote}
\small\ttfamily
You are an expert competitive programmer. Return only a complete Python 3 program that solves the problem. Do not use Markdown fences or explanations.\\[4pt]
\{problem description\}
\end{quote}

\codenetexamplesappendix
\clearpage

\section{Tabular intuition for finite-sample collision}
\label{sec:result2}

The classical per-context histogram provides a deliberately restricted comparison to the neural decomposition in
Section~\ref{sec:neural}. Fix a context $h$ with $n_h$ independent and identically distributed (i.i.d.) samples from $p_h$ and empirical maximum-likelihood estimate (MLE)
$\hat p_h(a)=n_{h,a}/n_h$. Then $\E[\hat p_h(a)]=p_h(a)$: finite samples favor no token in expectation,
although any realized dataset overweights some tokens.

\begin{proof}
Since $n_{h,a}\sim\mathrm{Binomial}(n_h,p_h(a))$, $\E[\hat p_h(a)]=\E[n_{h,a}]/n_h=p_h(a)$.
\end{proof}

Conditioning on one realized histogram makes its sampling fluctuations shared across all subsequent generations.

\noindent {\bf Simple example.} For $p=(1/3,1/3,1/3)$ and $n=6$, realized counts might be $(4,1,1)$ or
$(1,4,1)$. The favored token changes across datasets, but both histograms have higher collision than $p$.

\subsection{Expected collision inflation}
For arbitrary $p_h$, $\E[\hat p_h(a)^2] = p_h(a)^2 + \frac{p_h(a)(1-p_h(a))}{n_h}$. Summing over $a$,
\begin{equation}
\E[C(\hat p_h)] = C(p_h) + \frac{1-C(p_h)}{n_h},
\end{equation}
so the local expected collision ratio is
\begin{equation}
\frac{\E[C(\hat p_h)]}{C(p_h)} = 1 + \frac{1-C(p_h)}{n_h\,C(p_h)} = 1 + \frac{N_2(p_h)-1}{n_h}.
\end{equation}
This exact identity is the local scaling law for the tabular MLE:
\begin{equation}
\text{local expected collision ratio} = 1 + \frac{\text{target effective diversity } N_2(p_h) - 1}{\text{sample count } n_h}.
\end{equation}
The expected collision ratio is also exactly
$\E[N_2(p_h)/N_2(\hat p_h)]$, because $N_2(p_h)/N_2(\hat p_h)=C(\hat p_h)/C(p_h)$ for every
dataset. It generally differs from $N_2(p_h)/\E[N_2(\hat p_h)]$.

\begin{proof}[Derivation]
Fix the context $h$ and abbreviate $p(a)=p_h(a)$, $n=n_h$. The observed counts $(n_{h,a})_a$ are $\mathrm{Multinomial}(n,p_h)$; in particular each $n_{h,a}\sim\mathrm{Binomial}(n,p(a))$, so the empirical MLE $\hat p_h(a)=n_{h,a}/n$ has
\begin{equation}
\E[\hat p_h(a)]=p(a),
\qquad
\Var(\hat p_h(a))=\frac{p(a)\big(1-p(a)\big)}{n}.
\end{equation}
By the variance--mean-square identity $\E[X^2]=\Var(X)+(\E X)^2$,
\begin{equation}
\E[\hat p_h(a)^2]=\frac{p(a)\big(1-p(a)\big)}{n}+p(a)^2 .
\end{equation}
Summing over $a$ and using linearity of expectation (which needs no independence across tokens, so the within-multinomial correlations are irrelevant),
\begin{equation}
\E[C(\hat p_h)]=\sum_a \E[\hat p_h(a)^2]
=\sum_a p(a)^2+\frac1n\sum_a p(a)\big(1-p(a)\big)
=C(p_h)+\frac{1-C(p_h)}{n},
\end{equation}
where the last equality uses $\sum_a p(a)=1$, hence
$\sum_a p(a)(1-p(a))=1-\sum_a p(a)^2=1-C(p_h)$. Dividing by $C(p_h)$ and substituting
$N_2(p_h)=1/C(p_h)$,
\begin{equation}
\frac{\E[C(\hat p_h)]}{C(p_h)}
=1+\frac{1-C(p_h)}{n\,C(p_h)}
=1+\frac{N_2(p_h)-1}{n},
\end{equation}
because $(1-C(p_h))/C(p_h)=N_2(p_h)-1$. Restoring $n=n_h$ gives the stated law.
\end{proof}

The critical factor is not the overall size of the SFT dataset, but rather the local effective sample size, $n_h$, representing the number of fine-tuning examples that contribute meaningfully to context $h$.

\begin{example}[Why any finite sample inflates collision (Jensen)]\label{ex:jensen}
Collision $C(r)=\sum_a r(a)^2$ is convex. Since $\E[\hat p_h]=p_h$, Jensen's inequality gives
$\E[C(\hat p_h)]\ge C(p_h)$. The excess $(1-C(p_h))/n_h$ is the summed multinomial variance and vanishes
as $n_h\to\infty$.
\end{example}

\noindent {\bf Homogeneous-tree illustration.} If every context has a uniform $K=100$ continuation distribution,
each context has $n_h=1000$ independent observations, and continuation collision is identical across sibling
branches, the local factor is $1+99/1000=1.099$. Under these restrictive assumptions, $50$ decisions yield the
product $1.099^{50}\approx112$. This is an illustration, not a neural or generic sequence law.

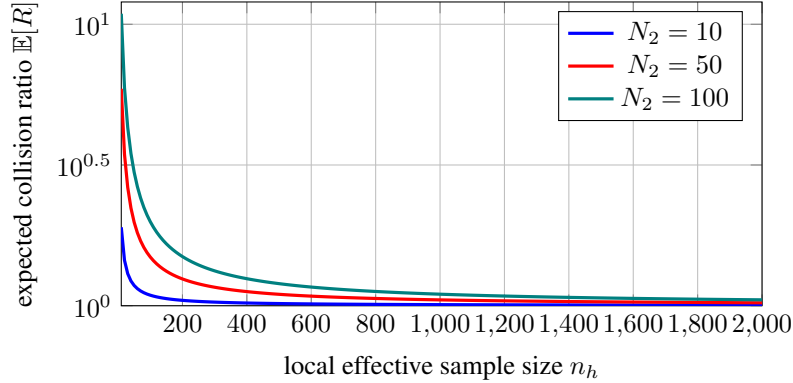
\begin{figure}[H]
\centering
\begin{tikzpicture}
\begin{axis}[width=0.72\textwidth,height=5.6cm,
  xlabel={local effective sample size $n_h$}, ylabel={expected collision ratio $\E[R]$},
  xmin=10,xmax=2000,ymode=log,ymin=1,ymax=12,domain=10:2000,samples=200,
  legend pos=north east,grid=both,every axis plot/.append style={very thick}]
\addplot[blue]{1+(10-1)/x};   \addlegendentry{$N_2=10$}
\addplot[red]{1+(50-1)/x};    \addlegendentry{$N_2=50$}
\addplot[teal]{1+(100-1)/x};  \addlegendentry{$N_2=100$}
\end{axis}
\end{tikzpicture}
\caption{\textbf{Tabular collision inflation decays with local coverage.} The exact expected collision ratio
$1+(N_2-1)/n_h$ for $N_2=10,50,100$.}
\label{fig:inflation}
\end{figure}

\begin{figure}[H]
\centering
\begin{tikzpicture}
\begin{axis}[width=0.72\textwidth,height=5.6cm,
  xlabel={sequence length $L$ (number of nontrivial decisions)}, ylabel={expected sequence collision ratio},
  ymode=log,xmin=0,xmax=60,domain=0:60,samples=61,
  legend pos=north west,grid=both,every axis plot/.append style={very thick}]
\addplot[teal]{(1+99/2000)^x};  \addlegendentry{$n_h=2000$}
\addplot[red]{(1+99/1000)^x};   \addlegendentry{$n_h=1000$}
\addplot[blue]{(1+99/500)^x};   \addlegendentry{$n_h=500$}
\end{axis}
\end{tikzpicture}
\caption{\textbf{Homogeneous-tree illustration.} With independent tabular estimates, $N_2=100$, and identical
continuation collision at every branch, the per-step factor compounds multiplicatively. The generic recursion need
not factor this way.}
\label{fig:compound}
\end{figure}
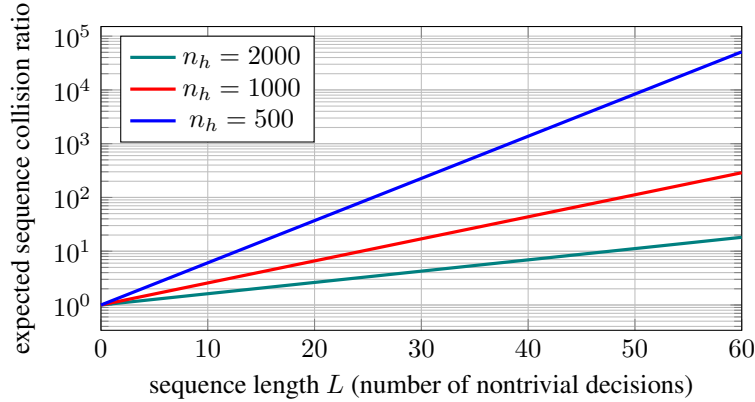

\begin{remark}[This is tabular, not neural]
\label{rem:tabular}
The identity above is exact for a per-context empirical MLE. A transformer shares parameters across contexts, so its
conditional need not have multinomial variance and no effective-count substitution restores this formula in
general. The neural experiments therefore use Eq.~\eqref{eq:ifbox} and Theorem~\ref{thm:kernelstability} instead.
\end{remark}

\subsection{Amplification over sequences}
For variable-length autoregressive models, sequence collision is not a simple product of probabilities over fixed positions, due to the dependence of contexts on sampled prefixes.  Analysis of these models therefore requires a branching recursive framework.
\begin{equation}
C_p(h) = \sum_a p(a\mid h)^2\, C_p(ha),
\label{eq:seqcollisionrec}
\end{equation}
Each summand is the probability that two draws both choose branch $a$ and then collide from $ha$. Thus the
collision-biased next-token weight is proportional to $p(a\mid h)^2C_p(ha)$, not to $p(a\mid h)^2$ alone; the
recursion is generally a sum of branch products rather than a product of aggregate one-step and continuation
collisions.

Let $D$ denote the random tabular dataset and let
$R_p(y\mid x)=p(y\mid x)^2/C_p(x)$ be the exact-match collision-biased path distribution. For small local terms
and approximately independent per-context estimates, a first-order expansion motivates the homogeneous-tree
heuristic
\begin{equation}
\log\frac{\E_D[C_{\hat p}(x)]}{C_p(x)} \approx \E_{Y\sim R_p(\cdot\mid x)}\!\left[\sum_{t=1}^{T(Y)+1} \frac{N_2(p_{h_t})-1}{n_{h_t}}\right],
\end{equation}
This approximation additionally treats $R_p$ as fixed and requires $C_p(ha)$ to be approximately constant across
sibling branches. Without those conditions, the $R_p$-average of local factors is not the expected sequence
collision ratio.

\subsection{Entropy analogue}
\label{sec:entropy}
For a fixed finite support of size $S$, the classical plug-in bias is
$\E[H(\hat p_h)]=H(p_h)-(S-1)/(2n_h)+o(n_h^{-1})$ under its regularity conditions
\citep{paninski2003,jiao2015}. Replacing $S$ by $\exp H(p_h)$ or summing this local approximation along generated
paths is heuristic; neither substitution is used in the empirical claims.

\section{Limits of temperature scaling}
\label{sec:temperature}
Temperature can change a model's collision value without generally calibrating its distribution. To make the global
sequence transformation well-defined for every $T>0$, fix a distribution $q(\cdot\mid x)$ with finite support and
write $\beta=1/T$, $q_T(y\mid x)=q(y\mid x)^\beta/Z_\beta(x)$, and
$Z_\beta(x)=\sum_yq(y\mid x)^\beta$.

\begin{proposition}[Global temperature monotonically increases order-2 diversity]
\begin{equation}
C_T(x) = \sum_y q_T(y\mid x)^2 = \frac{Z_{2\beta}(x)}{Z_\beta(x)^2} = \frac{Z_{2/T}(x)}{Z_{1/T}(x)^2}.
\end{equation}
Moreover, $C_T(x)$ is nonincreasing and $N_2^T(x)=1/C_T(x)$ is nondecreasing in $T$.
\end{proposition}
\begin{proof}
With $q_T(y)=q(y)^\beta/Z_\beta$ and $Z_\beta=\sum_y q(y)^\beta$, $C_T=\sum_y q_T(y)^2=\sum_y q(y)^{2\beta}/Z_\beta^2=Z_{2\beta}/Z_\beta^2$. For monotonicity in $T$, let $\psi(\beta)=\log Z_\beta$. Then $\psi'(\beta)=\E_{q_\beta}[\log q]$ where $q_\beta\propto q^\beta$, and $\psi''(\beta)=\Var_{q_\beta}(\log q)\ge0$, so $\psi'$ is nondecreasing. Since $\log C_T=\psi(2\beta)-2\psi(\beta)$, $\frac{d}{d\beta}\log C_T=2\big[\psi'(2\beta)-\psi'(\beta)\big]\ge0$ for $\beta>0$. Thus $C_T$ is nondecreasing in $\beta=1/T$, i.e.\ nonincreasing in $T$, with equality only in the degenerate cases where $\log q$ is $q_\beta$-a.s.\ constant (uniform or point mass); equivalently the order-2 diversity $N_2^{T}(x)=1/C_T(x)$ is nondecreasing in $T$.
\end{proof}
The reachable values depend only on the fixed probability profile of $q$.

\subsection{The temperature-reachable diversity interval}
Raising global sequence temperature increases diversity, but only within a fixed interval determined by the model distribution. For a given prompt $x$, let $S(x)$ denote the number of sequences with non-zero probability under $q$ given $x$ ($S(x):=|\mathrm{supp}\,q(\cdot\mid x)|$), and let $g_\star(x)$ denote the number of modes of $q$ given $x$ ($g_\star(x):=|\arg\max_y q(y\mid x)|$).

\begin{lemma}[Temperature-reachable diversity interval]
For a fixed model distribution $q(\cdot\mid x)$ with finite support, $N_2^{T}(x)=Z_\beta(x)^2/Z_{2\beta}(x)$ is continuous and nondecreasing in $T$, with limits
\begin{equation}
\lim_{T\to 0^+} N_2^{T}(x)=g_\star(x),
\qquad
\lim_{T\to\infty} N_2^{T}(x)=S(x).
\end{equation}
Hence as $T$ ranges over $(0,\infty)$, $N_2^{T}(x)$ takes every value in the open interval $(\,g_\star(x),\,S(x)\,)$; the endpoints $g_\star(x)$ and $S(x)$ are the limits as $T\to0^+$ and $T\to\infty$ and are not attained at any finite $T$ except in degenerate cases ($q$ already a point mass, resp.\ already uniform on its support). Its closure is $[\,g_\star(x),\,S(x)\,]$ (generically $[1,S(x)]$). One degenerate case also affects the lower endpoint: if $q(\cdot\mid x)$ is uniform on its entire support, then every token is a maximizer, so $g_\star(x)=S(x)$ and $N_2^{T}(x)=S(x)$ at every finite $T$; the reachable interval then collapses to the single point $\{S(x)\}$. The ``point mass / uniform limit'' description above is thus the generic non-degenerate picture, but is incomplete for this collapsed (uniform-$q$) case.
\end{lemma}
\begin{proof}
Monotonicity is the previous result ($C_T$ nonincreasing in $T$, so $N_2^T=1/C_T$ nondecreasing). For $T\to0^+$ ($\beta\to\infty$): writing $q_{\max}=\max_y q(y)$ and $g_\star$ for the number of maximizers, $Z_\beta=g_\star q_{\max}^\beta(1+o(1))$ and $Z_{2\beta}=g_\star q_{\max}^{2\beta}(1+o(1))$, so $N_2^T=Z_\beta^2/Z_{2\beta}\to g_\star^2 q_{\max}^{2\beta}/(g_\star q_{\max}^{2\beta})=g_\star$. For $T\to\infty$ ($\beta\to0$): $q(y)^\beta\to\mathbf 1\{q(y)>0\}$, so $Z_\beta\to S$ and $Z_{2\beta}\to S$, giving $N_2^T\to S^2/S=S$. Continuity in $T$ and the intermediate value theorem give every value in the open interval; the endpoints are limits, attained only in the stated degenerate cases.
\end{proof}

When $S(x)>g_\star(x)$, the fraction of this interval reached at $T$ is
\begin{equation}
\rho_T(x):=\frac{N_2^{T}(x)-g_\star(x)}{S(x)-g_\star(x)}\in[0,1],
\end{equation}
which is a normalized diversity location, not a temperature parameter. It is undefined when the interval collapses.

\begin{theorem}[Temperature only redistributes existing support]
\label{thm:tempsupport}
Assume a finite candidate space (e.g.\ a finite vocabulary together with a maximum sequence length), so that $S(x)=|\mathrm{supp}\,q(\cdot\mid x)|<\infty$. For every temperature $T>0$ the temperature-scaled model $q_T$ has the same support as $q$, and its diversity is bounded by that support size:
\begin{equation}
N_2^{T}(x)\le S(x),
\qquad
\sup_{T>0}N_2^{T}(x)=S(x),
\end{equation}
the supremum corresponding to the uniform distribution on $\mathrm{supp}\,q(\cdot\mid x)$ and approached only as $T\to\infty$ (not attained at any finite $T$ unless $q$ is already uniform on its support). Consequently, relative to the true conditional $p(\cdot\mid x)$:
\begin{enumerate}[label=(\roman*)]
\item if $N_2\big(p(\cdot\mid x)\big)>S(x)$, no temperature reaches the true diversity---the model's support is simply too small;
\item if $N_2\big(p(\cdot\mid x)\big)$ lies in the open interval $(g_\star(x),S(x))$, some finite temperature attains the true diversity value by the intermediate value theorem; the boundary values $g_\star(x)$ and $S(x)$ are only approached as $T\to0^+$ and $T\to\infty$ and are not attained at any finite $T$ except in the degenerate cases of the interval Lemma; but
\item attaining that value does not reproduce $p$: $q_T=p$ holds only if $p(\cdot\mid x)\propto q(\cdot\mid x)^{\beta}$ for some $\beta>0$.
\end{enumerate}
\end{theorem}
\begin{proof}
Temperature scaling preserves support: $q_T(y)>0\iff q(y)>0$ for every finite $T>0$, so $q_T$ is supported on a set of size $S(x)$. For any distribution $r$ over $S$ outcomes, Cauchy--Schwarz gives $1=\big(\sum_a r(a)\big)^2\le S\sum_a r(a)^2$, i.e.\ $N_2(r)=1/\sum_a r(a)^2\le S$, with equality iff $r$ is uniform on its support. Applying this to $r=q_T$ yields $N_2^{T}(x)\le S(x)$; the bound is approached as $q_T\to\mathrm{Unif}(\mathrm{supp}\,q)$ when $T\to\infty$ and is not attained at finite $T$ unless $q$ is already uniform on its support. This gives the displayed bounds and case~(i): if $N_2(p)>S(x)$, no $q_T$ can reach it. For case~(ii), $N_2^{T}(x)$ is continuous in $T$ and takes every value in the open interval $(g_\star(x),S(x))$ (Lemma), so whenever $N_2(p)$ lies in this open interval the intermediate value theorem supplies a finite $T$ with $N_2^{T}(x)=N_2(p)$; the endpoints are limits, not attained at finite $T$ outside the degenerate cases. For case~(iii), $\{q_T:T>0\}$ is the one-parameter family $\log q_T(y)=\beta\log q(y)-\log Z_\beta$; matching $p$ exactly requires $\log p(y)=\beta\log q(y)-\log Z_\beta$ for all $y$, i.e.\ $p\propto q^{\beta}$.
\end{proof}

\noindent {\bf Interpretation.}
Temperature preserves support and moves the global sequence distribution from the uniform distribution over its
maximizers as $T\to0^+$ toward the uniform distribution over its support as $T\to\infty$. It can therefore match a
target collision value while missing the target distribution, and the high-temperature limit recovers the target
only when the target itself is uniform on the model's support.

\begin{lemma}[Temperature scaling is an order-preserving one-parameter family]
\label{lem:orderpreserving}
Fix a prompt $x$ and write $q_T(\cdot\mid x)\propto q(\cdot\mid x)^{1/T}$.
\begin{enumerate}[label=(\roman*)]
\item Log-odds identity. For every $T>0$ and any outcomes $a,b$ with $q(b\mid x)>0$,
\begin{equation}
\log\frac{q_T(a\mid x)}{q_T(b\mid x)}=\frac{1}{T}\,\log\frac{q(a\mid x)}{q(b\mid x)} .
\end{equation}
\item Order preservation. Hence temperature scaling is an order-preserving reweighting: the rank ordering of
$\{q_T(y\mid x)\}_y$ is the same for all $T>0$. The reachable set is a one-parameter curve from the uniform
distribution over $\argmax_y q(y\mid x)$ as $T\to0^+$ to the uniform distribution over
$\operatorname{supp}q(\cdot\mid x)$ as $T\to\infty$.
\item Unreachability. Consequently $q_T=p(\cdot\mid x)$ for some $T>0$ only if $p(\cdot\mid x)\propto q(\cdot\mid x)^{\beta}$, i.e.\ $\log p$ is an affine function of $\log q$. In particular, if the rank ordering of $p(\cdot\mid x)$ differs from that of $q(\cdot\mid x)$, no temperature matches $p$.
\end{enumerate}
\end{lemma}
\begin{proof}
For (i), $q_T(y)=q(y)^{1/T}/Z_{1/T}$, so the normalizer cancels in the ratio and $q_T(a)/q_T(b)=(q(a)/q(b))^{1/T}$; take logarithms. For (ii), $u\mapsto u^{1/T}$ is strictly increasing on $(0,\infty)$, so $q_T(a)\ge q_T(b)$ exactly when $q(a)\ge q(b)$; this common ordering holds for all $T$, with the two endpoints given by the interval Lemma above. For (iii), a fixed ordering rules out any $q_T$ with a different one, while $q_T=p$ forces $p\propto q^{1/T}$.
\end{proof}

Temperature rescales every log-odds ratio by the same factor. It can flatten or sharpen relative-probability errors,
but it cannot reorder outcomes or independently correct them. Figure~\ref{fig:tempbars} gives a six-outcome example.

\subsection{Effective support}
The finite-support Theorem~\ref{thm:tempsupport} requires a finite candidate space to ensure $S(x)<\infty$. While standard softmax assigns strictly positive probability to all tokens and, subject to a length constraint, to most sequences, the Theorem's condition of finite support may be violated without a maximum length, rendering the claim about temperature's inability to recover missing support inapplicable.  A more broadly applicable statement is probabilistic: rare, valid continuations are not entirely absent from $q_T$, but are effectively absent when sampling from a feasible number of samples.

\begin{proposition}[Effective absence of valid missing mass under sampling]
\label{prop:effsupport}
Fix a prompt $x$ and consider a sequence of fitted models $q_n$ indexed by dataset size $n$, with temperature-scaled tail mass $\varepsilon_{n,T}(x)=q_{n,T}(U(x)\mid x)$ on a set $U(x)$ of valid but unobserved or severely underweighted continuations. Drawing $M_n$ independent generations from $q_{n,T}(\cdot\mid x)$,
\begin{equation}
\Pr\big(\text{at least one of the }M_n\text{ samples lands in }U(x)\big)=1-\big(1-\varepsilon_{n,T}(x)\big)^{M_n}\le M_n\,\varepsilon_{n,T}(x).
\end{equation}
Hence if $\varepsilon_{n,T}(x)\to0$ and $M_n\varepsilon_{n,T}(x)\to0$ along a sequence of fitted models,
these continuations are effectively absent from the generated samples. This is an assumed asymptotic regime over
models, not a consequence of changing $T$ at fixed $q_n$.
\end{proposition}
\begin{proof}
The $M_n$ generations are i.i.d.\ draws from $q_{n,T}(\cdot\mid x)$, so the probability that none lands in $U(x)$ is $\big(1-\varepsilon_{n,T}(x)\big)^{M_n}$, giving the stated equality. Bernoulli's inequality $(1-\varepsilon)^{M_n}\ge1-M_n\varepsilon$ for $\varepsilon\in[0,1]$ yields $1-(1-\varepsilon_{n,T}(x))^{M_n}\le M_n\,\varepsilon_{n,T}(x)$. If $M_n\,\varepsilon_{n,T}(x)\to0$ this probability vanishes, so with high probability no sample realizes any continuation in $U(x)$.
\end{proof}

\noindent {\bf No validity-aware guarantee.}
Split an underweighted tail into valid and invalid sets, $U_{\mathrm{valid}}(x)$ and
$U_{\mathrm{invalid}}(x)$. Temperature applies the same map $u\mapsto u^{1/T}$ to both:
\begin{equation}
\label{eq:valinvalratio}
\frac{q_T(U_{\mathrm{valid}}(x)\mid x)}{q_T(U_{\mathrm{invalid}}(x)\mid x)}=\frac{\sum_{y\in U_{\mathrm{valid}}(x)}q(y\mid x)^{1/T}}{\sum_{y\in U_{\mathrm{invalid}}(x)}q(y\mid x)^{1/T}}.
\end{equation}
The ratio can vary with $T$, but the transformation contains no validity signal. Temperature can incidentally increase
valid mass; it cannot guarantee selective recovery. Proposition~\ref{prop:effsupport} implies effective absence only
when its stated condition $M_n\varepsilon_{n,T}(x)\to0$ holds.

\begin{example}[A two-peaked target temperature cannot reach]\label{ex:twopeak}
Take a fixed six-outcome context in which the model $q=(0.50,0.25,0.12,0.07,0.04,0.02)$ is monotone decreasing while the target $p=(0.30,0.05,0.05,0.05,0.25,0.30)$ is two-peaked: here $N_2(p)=4\le S=6$, so the diversity value is attainable (Theorem, case (ii)), yet the shape of $p$ is not, because its ordering differs from that of $q$ (case (iii)).
\end{example}

\begin{figure}[!ht]
\centering
\begin{tikzpicture}
\begin{groupplot}[
  group style={group size=4 by 1, horizontal sep=0.25cm,
    ylabels at=edge left, yticklabels at=edge left},
  width=0.245\textwidth, height=4.8cm,
  ybar, /pgf/bar width=4pt, enlarge x limits=0.16,
  ymin=0, ymax=0.8,
  xtick={1,2,3,4,5,6}, xlabel={outcome},
  xticklabel style={font=\footnotesize}, title style={font=\small},
]
\nextgroupplot[title={$T=0.5$}, ylabel={probability},
  legend style={at={(0.97,0.97)}, anchor=north east, font=\footnotesize,
    draw=none, fill=none}]
\addplot[draw=blue!75, fill=blue!40] coordinates {(1,0.7490)(2,0.1872)(3,0.0431)(4,0.0147)(5,0.0048)(6,0.0012)};
\addplot[red, thick, mark=*, sharp plot, /pgf/bar shift=0pt] coordinates {(1,0.30)(2,0.05)(3,0.05)(4,0.05)(5,0.25)(6,0.30)};
\addlegendentry{$q_T$}
\addlegendentry{$p$}
\nextgroupplot[title={$T=1$}]
\addplot[draw=blue!75, fill=blue!40] coordinates {(1,0.5000)(2,0.2500)(3,0.1200)(4,0.0700)(5,0.0400)(6,0.0200)};
\addplot[red, thick, mark=*, sharp plot, /pgf/bar shift=0pt] coordinates {(1,0.30)(2,0.05)(3,0.05)(4,0.05)(5,0.25)(6,0.30)};
\nextgroupplot[title={$T=2$}]
\addplot[draw=blue!75, fill=blue!40] coordinates {(1,0.3274)(2,0.2315)(3,0.1604)(4,0.1225)(5,0.0926)(6,0.0655)};
\addplot[red, thick, mark=*, sharp plot, /pgf/bar shift=0pt] coordinates {(1,0.30)(2,0.05)(3,0.05)(4,0.05)(5,0.25)(6,0.30)};
\nextgroupplot[title={$T=\infty$}]
\addplot[draw=blue!75, fill=blue!40] coordinates {(1,0.1667)(2,0.1667)(3,0.1667)(4,0.1667)(5,0.1667)(6,0.1667)};
\addplot[red, thick, mark=*, sharp plot, /pgf/bar shift=0pt] coordinates {(1,0.30)(2,0.05)(3,0.05)(4,0.05)(5,0.25)(6,0.30)};
\end{groupplot}
\end{tikzpicture}
\caption{\textbf{Temperature preserves outcome order.} For
$q=(0.50,0.25,0.12,0.07,0.04,0.02)$, increasing $T$ flattens the distribution but cannot reproduce the off-order
target $p=(0.30,0.05,0.05,0.05,0.25,0.30)$, even though $N_2(p)=4<S=6$.}
\label{fig:tempbars}
\end{figure}
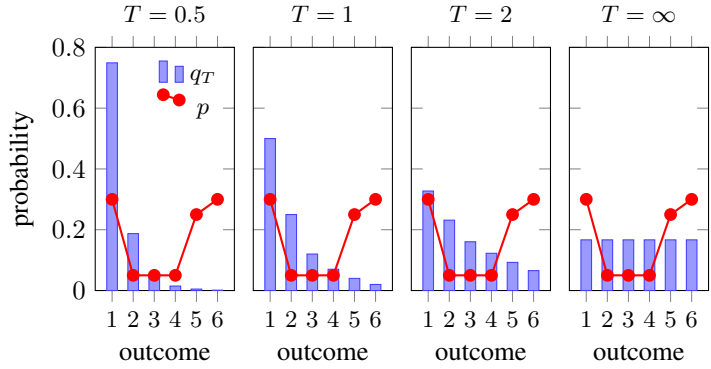

\noindent {\bf Two-sided miscalibration.}
Temperature can leave diversity below the target or flatten beyond it. Neither direction alone establishes output
validity: low-probability continuations can be valid or invalid. Even when a temperature matches $N_2(p)$, it
matches $p$ only in the power-family case of Theorem~\ref{thm:tempsupport}.

The rate of temperature-driven diversity change is determined by the spread of the model's log-probabilities.  Defining $\psi(\beta) = \log Z_\beta$, we have $\psi'(\beta) = \E_{q_\beta}[\log q]$ and $\psi''(\beta) = \Var_{q_\beta}(\log q) \ge 0$, where $q_\beta$ is proportional to $q^\beta$. Given that $\log C_T = \psi(2\beta) - 2\psi(\beta)$, the derivative with respect to $\log T$ is $-\beta$ times the derivative with respect to $\beta$.
\begin{equation}
\frac{d\,\log N_2^{T}(x)}{d\,\log T}
= -\frac{d\,\log C_T}{d\,\log T}
= 2\beta\!\int_{\beta}^{2\beta}\Var_{q_s}(\log q)\,ds \;\ge\;0 .
\end{equation}
The derivative is governed by the log-probability variance along the power family. For nonuniform finite-support
$q$, finite temperatures trace the open interval $(g_\star(x),S(x))$ and approach its endpoints only in the limits.

\begin{figure}[H]
\centering
\begin{tikzpicture}
\begin{axis}[width=0.72\textwidth,height=5.6cm,
  xlabel={decoding temperature $T$ (log scale)},
  ylabel={effective number of choices $N_2^{T}$},
  xmode=log,xmin=0.1,xmax=30,ymin=0.8,ymax=4.2,
  samples=200,grid=both,legend pos=south east,
  every axis plot/.append style={very thick}]
\addplot[blue,domain=0.1:30]
  { (0.7^(1/x)+0.2^(1/x)+0.07^(1/x)+0.03^(1/x))^2 / (0.7^(2/x)+0.2^(2/x)+0.07^(2/x)+0.03^(2/x)) };
\addlegendentry{$N_2^{T}$ for $q=(0.7,0.2,0.07,0.03)$}
\addplot[black,densely dashed,domain=0.1:30]{1};
\addplot[black,densely dashed,domain=0.1:30]{4};
\addplot[red,only marks,mark=*] coordinates {(1,1.87)};
\node[anchor=west,font=\footnotesize] at (axis cs:1.05,1.87){$T=1:\ N_2\approx1.87$};
\node[anchor=south west,font=\footnotesize] at (axis cs:0.11,4){$S=4$};
\node[anchor=north west,font=\footnotesize] at (axis cs:0.11,1){$g_\star=1$};
\end{axis}
\end{tikzpicture}
\caption{\textbf{Global temperature traces a fixed diversity interval.} For
$q=(0.7,0.2,0.07,0.03)$, $N_2^T$ increases from the mode count $1$ toward the support size $4$; finite
temperatures attain neither endpoint.}
\label{fig:tempinterval}
\end{figure}
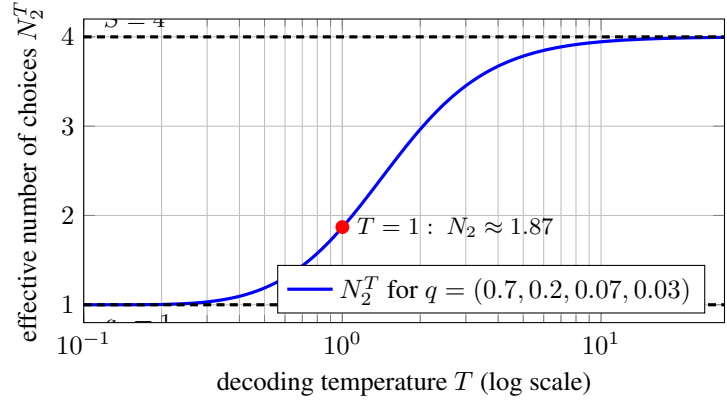

LLM decoding employs temperature locally, applying $q_T(a\mid h) = q(a\mid h)^{1/T} / \sum_b q(b\mid h)^{1/T}$, which influences the future contexts reached, as well as the probabilities of EOS, token length, format, and refusal, and the mixture of rare-valid and rare-invalid tokens. The monotonicity result for global sequence temperature therefore need not hold for token-by-token decoding. \citet{arnold2024phasetransitionsoutputdistribution} observe this breakdown empirically: because token-by-token decoding tempers each next-token distribution rather than the sequence distribution, the mean sequence-level energy $\E_{q_T}[-\log q]$ can fall as $T$ rises (a negative ``heat capacity''), which $\psi''\ge 0$ rules out under global temperature.

\noindent {\bf Simple example.} Suppose the root probabilities are $0.4$ for EOS and $0.6$ for CONTINUE, followed
by $100$ uniform endings after CONTINUE. Sequence collision is
$0.4^2+0.6^2(0.01)=0.1636$. Raising local temperature moves the root toward $(0.5,0.5)$ while leaving the
uniform branch unchanged, so collision approaches $0.5^2+0.5^2(0.01)=0.2525$. Thus local temperature can reduce
sequence diversity even though global temperature of a fixed sequence distribution cannot.

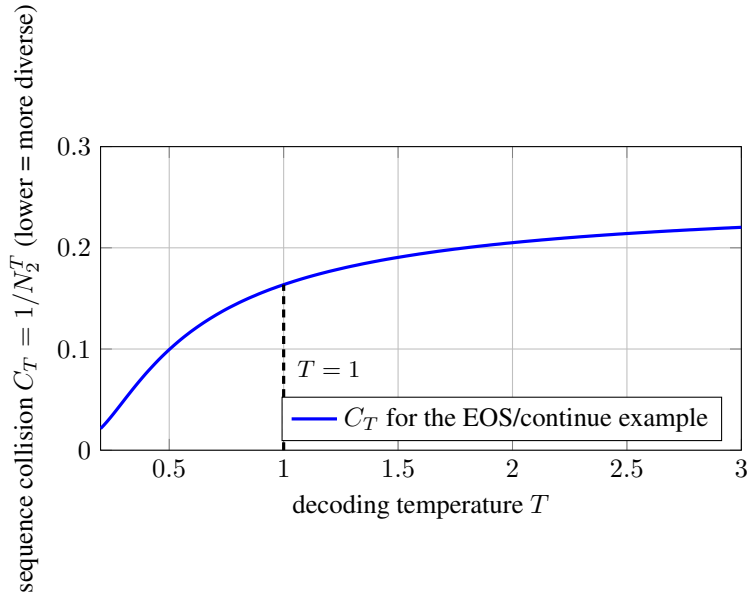
\begin{figure}[H]
\centering
\begin{tikzpicture}
\begin{axis}[width=0.72\textwidth,height=5.6cm,
  xlabel={decoding temperature $T$}, ylabel={sequence collision $C_T=1/N_2^{T}$ (lower = more diverse)},
  xmin=0.2,xmax=3,ymin=0,ymax=0.3,samples=200,grid=both,
  legend pos=south east,every axis plot/.append style={very thick}]
\addplot[blue,domain=0.2:3]
  { (0.4^(1/x))^2/((0.4^(1/x)+0.6^(1/x))^2) + 0.01*(0.6^(1/x))^2/((0.4^(1/x)+0.6^(1/x))^2) };
\addlegendentry{$C_T$ for the EOS/continue example}
\addplot[black,densely dashed] coordinates {(1,0)(1,0.1636)};
\node[anchor=west,font=\footnotesize] at (axis cs:1.02,0.08){$T=1$};
\end{axis}
\end{tikzpicture}
\caption{\textbf{Local temperature can increase sequence collision.} In the two-stage EOS/continue example,
raising $T$ moves mass toward the short deterministic branch and reduces sequence diversity.}
\label{fig:temperature}
\end{figure}

\end{document}